\documentclass{article}

\usepackage[final]{preamble/neurips_2020}

\usepackage[utf8]{inputenc} %
\usepackage[T1]{fontenc}    %
\usepackage{nicefrac}       %
\usepackage{microtype}      %

\usepackage{booktabs,tabularx}

\usepackage{algorithm}
\usepackage{algorithmic}

\usepackage{amsmath,amsthm,amssymb,amsfonts,mathtools}
\newtheorem{proposition}{Proposition}

\newtheorem{lemma}{Lemma}

\usepackage{bm}
\usepackage{scalefnt}
\usepackage{letltxmacro}

\usepackage{graphicx}
\usepackage[labelfont=bf]{caption}
\usepackage[format=hang]{subcaption}
\usepackage{wrapfig}

\usepackage{natbib}
\usepackage[acronym,xindy]{glossaries}

\newacronym{ds}{ds}{Deep Set}
\newacronym{st}{st}{Set Transformer}
\newacronym{nc}{NC}{Neural Complexity}
\newacronym{ns}{NS}{Neural Statistician}
\newacronym{np}{NP}{Neural Process}

\usepackage[dvipsnames]{xcolor}

\usepackage{tikz}
\usetikzlibrary{bayesnet}

\usepackage{hyperref}
\hypersetup{hidelinks, linktoc=all, hypertexnames=true, colorlinks=true}
\hypersetup{urlcolor=BrickRed!90!Sepia, linkcolor=BrickRed!90!Sepia, citecolor=Aquamarine!30!Blue} 
\usepackage[capitalize,nameinlink,noabbrev]{cleveref} %
\creflabelformat{equation}{#2\textup{#1}#3}  %

\usepackage{pifont}%

\newcommand{\calH}{{\mathcal{H}}}

\newcommand{\calL}{{\mathcal{L}}}

\newcommand{\calX}{{\mathcal{X}}}
\newcommand{\calY}{{\mathcal{Y}}}
\newcommand{\calZ}{{\mathcal{Z}}}

\newcommand{\Real}{\mathbb R}

\newcommand{\Prob}{\mathbb P}

\newcommand{\E}{\mathbb E}

\newcommand{\defeq}{\stackrel{\mathrm{def}}{=}}

\newcommand{\bracket}[1]{\left[ #1 \right]}

\DeclarePairedDelimiterX{\infdivx}[2]{(}{)}{#1\;\delimsize\|\;#2}

\renewrobustcmd{\bfseries}{\fontseries{b}\selectfont}
\renewrobustcmd{\boldmath}{}
\newrobustcmd{\BF}{\bfseries}

\newcommand{\msd}[2]{{#1}{\tiny $\pm${#2}}}

\newcommand{\Ltrue}{\calL_T}
\newcommand{\Lemp}{\widehat{\calL}_{T,S}}

\newcommand{\Lnc}{\calL_{NC}}
\newcommand{\gap}{G_{T,S}}

\title{Neural Complexity Measures}
\author{%
  Yoonho Lee$^{1}$,
  Juho Lee$^{1, 2}$,
  Sung Ju Hwang$^{1, 2}$,
  Eunho Yang$^{1, 2}$,
  Seungjin Choi$^{3}$\\
  AITRICS$^1$, Seoul, South Korea, KAIST$^2$, Daejeon, South Korea, BARO AI$^3$, Seoul, South Korea \\
  \texttt{eddy@aitrics.com} \\
}

\begin{document}
\maketitle

\begin{abstract}
    While various complexity measures for deep neural networks exist, specifying an appropriate measure capable of predicting and explaining generalization in deep networks has proven challenging. We propose \textit{Neural Complexity} (NC), a meta-learning framework for predicting generalization. Our model learns a scalar complexity measure through interactions with many heterogeneous tasks in a data-driven way. The trained NC model can be added to the standard training loss to regularize any task learner in a standard supervised learning scenario. We contrast NC's approach against existing manually-designed complexity measures and other meta-learning models, and we validate NC's performance on multiple regression and classification tasks.
\end{abstract}

\section{Introduction}
\label{sec:intro}
Deep neural networks have achieved excellent performance on numerous tasks, including image classification \citep{krizhevsky2012imagenet} and board games \citep{silver2017mastering}. Although they achieve superior empirical performance, why and how these models generalize remains a mystery. Thus, understanding which properties of deep networks allow them to generalize an important problem with far-reaching potential benefits such as principled model design and safety-aware models. To explain why deep networks generalize in practice, recent works have proposed novel \textit{complexity measures} for deep networks \citep{jiang2019fantastic,keskar2016large,liang2017fisher,DBLP:journals/corr/abs-1901-01672}.
Such measures quantify the complexity of the function that a neural network represents.
Ideally, such complexity measures should be good predictors of the degree of generalization of a network.
However, in practice, such manually-designed complexity measures have failed to capture essential properties of generalization in deep networks, such as improving with network size and worsening with label noise.

To overcome such limitations, we propose an alternative data-driven approach for constructing a complexity measure.
Our model, \gls{nc}, meta-learns a neural network that takes a predictor as input and outputs a scalar.
Similarly to previous complexity-based generalization bounds, we provide a probabilistic bound of the true loss using \gls{nc}. 
Our bound has very different characteristics from previous generalization bounds:
it depends on both data distribution and architecture, and more importantly, becomes tighter as the \gls{nc} model improves. 

Experimentally, we show that a learned \gls{nc} model consistently accelerates training in addition to preventing overfitting.
We also show the degree to which the measure learned by \gls{nc} transfers to different hypothesis classes, such as using a different network architecture, learning rate, or nonlinearity for the task learner.
Compared to other recent meta-learning methods \citep{wu2018understanding}, the meta-learned knowledge in \gls{nc} is much more stable across long learning trajectories.
Finally, while most meta-learning works focus on improving performance on small tasks such as few-shot classification, we show that \gls{nc} is also capable of regularizing learning in single large tasks.
\section{Problem Setup}
\label{section:formulation}

We adopt a meta-learning problem formulation in which a model (the "meta-learner") facilitates learning in new tasks using previous experience learning in other related tasks.
Specifically, we assume all tasks share sample space $\calZ$, hypothesis space $\calH$, and loss function $\calL: \calH \times \calZ \rightarrow \Real$. 
Each task $T$ consists of i.i.d. sampled finite training set $S = \{ z_1, \ldots, z_m \}$ from the underlying hidden distribution $D_T$ over the sample space $\calZ$ associated with task $T$.
The \textit{true loss} $\Ltrue$ and \textit{empirical loss} $\Lemp$ for each task $T$ are respectively defined as
\begin{align}
    \Ltrue(h) \defeq \E_{z \sim D_T} \bracket{\calL(h, z)} \quad \text{and} \quad
    \Lemp(h) \defeq \frac{1}{m} \sum_{z \in S} \calL(h, z).
\end{align}
Tasks themselves are sampled i.i.d. from a distribution of tasks: $T \sim \tau$. 
The objective of our meta-learner is to predict the difference between the true and empirical losses, otherwise known as the \textit{generalization gap} $\gap$:
\begin{align}
    \label{eq:gap_def}
    \gap(h) = \Ltrue(h) - \Lemp(h).
\end{align}
In other words, our model meta-learns a mapping $\calH \rightarrow \Real$ which mimics $h \mapsto \gap(h)$ by observing $\Ltrue(h)$ and $\Lemp(h)$ in many different tasks that follow $T \sim \tau$.

Even in the usual single-task supervised learning setup, we can still use this problem formulation to meta-learn by constructing a set of tasks in the following way.
Given one large dataset $S = \{ z_1, \ldots, z_M \}$, we randomly split $S$ into disjoint training and validation sets. 
For each task with this random split, the task learner uses the train set to train $h$, and the meta-learner evaluates $\Ltrue$ computed with the validation set as its target.
After training a meta-learner on this simulated set of tasks, we use the same model to estimate the gap $\gap$ of the full dataset $S$.
This task-splitting scheme is similar to traditional cross-validation. However, instead of choosing among a few hyperparameters, we meta-learn a neural network to mimic the complex mapping $h \mapsto \gap(h)$.
\section{Neural Complexity}
\glsreset{nc}
\begin{figure}[t]
\begin{subfigure}{.35\textwidth}
    \centering \includegraphics[width=\linewidth]{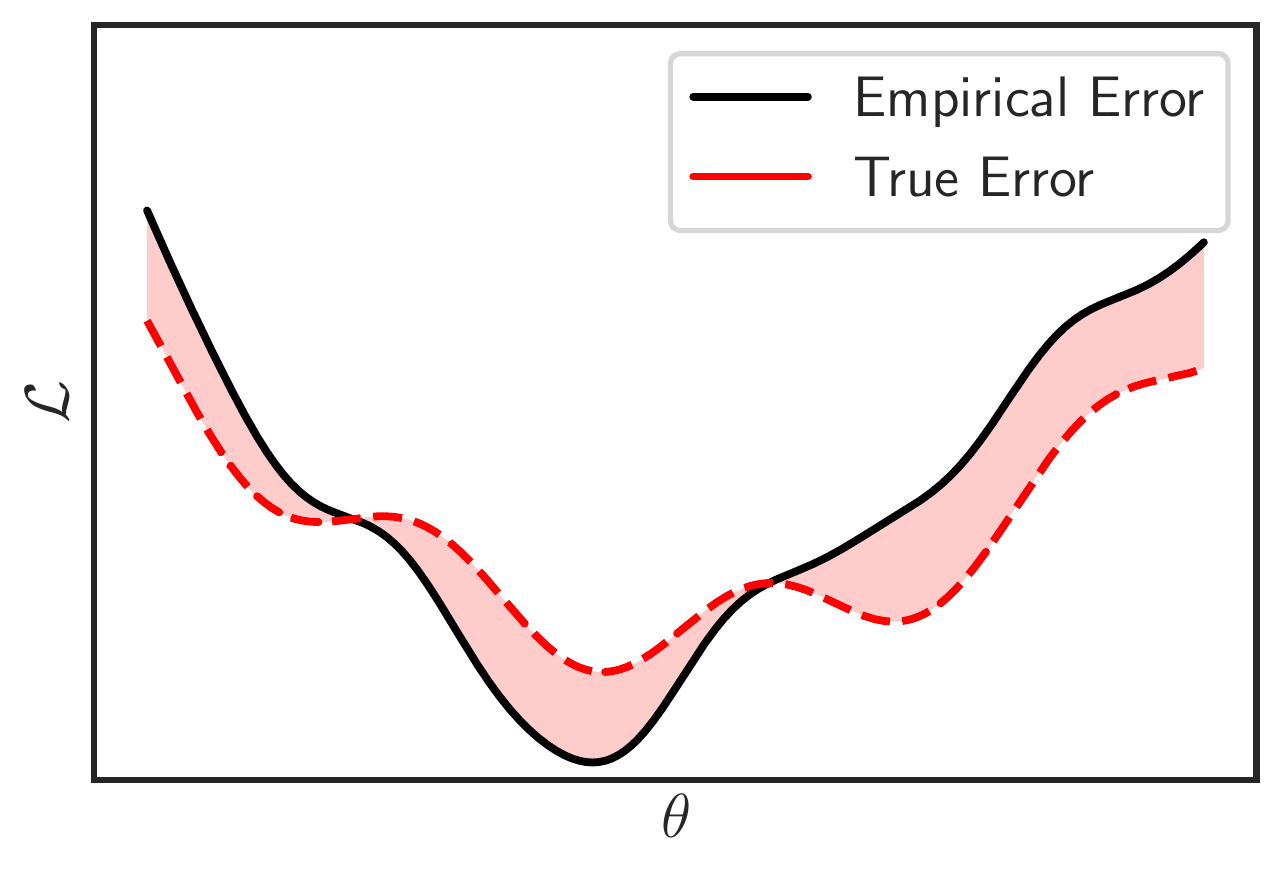}  
\end{subfigure}
\hfill 
\begin{subfigure}{.6\textwidth}
\centering\includegraphics[width=\linewidth]{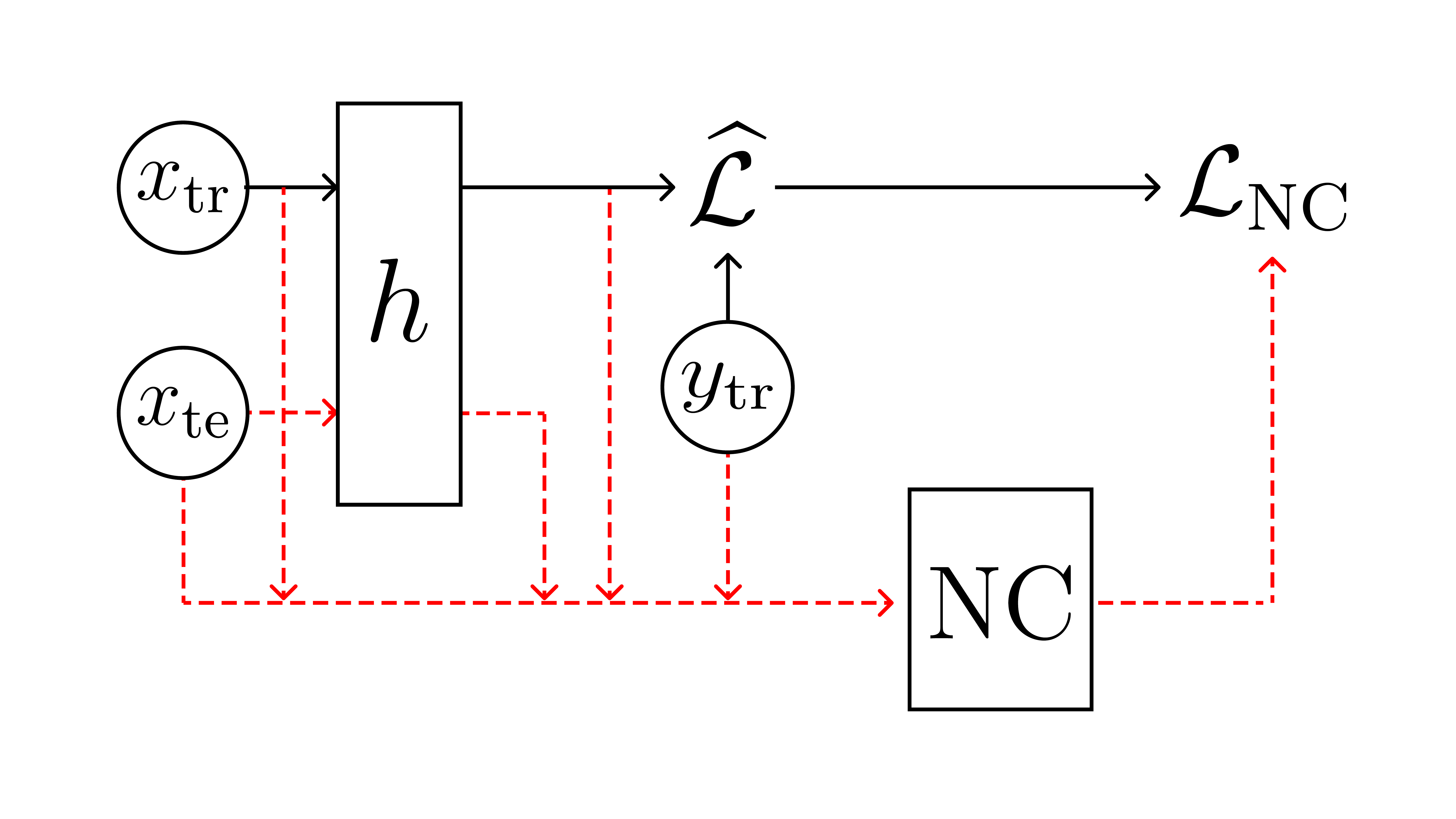}
\end{subfigure}
\vspace{-15pt}
\caption{
    (left) The true and empirical losses are correlated but different.
    \gls{nc} estimates their difference (colored).
    (right) The training loss $\widehat{\mathcal{L}}$ is regularized (solid lines) by the output of the trained \gls{nc} model (dotted lines).
    NC is meta-learned so that $\Lnc$ mimics the test loss.
    }
\label{fig:overview}
\end{figure}
\label{sec:formulation}
We now describe \glsentryfull{nc}, a meta-learning framework for predicting generalization.
At the core of \gls{nc} is a neural network that directly meta-learns a complexity measure through interactions with many tasks.
We show how this network integrates with any standard task learner in \cref{fig:overview}.
We show \gls{nc}'s training loop in \cref{fig:concept} and also provide a detailed description in \cref{alg:task_learning_with_nc,alg:meta_learning_nc}.

\begin{algorithm}
\caption{Task Learning}
\label{alg:task_learning_with_nc}
\begin{algorithmic}
    \REQUIRE \gls{nc}, Train and test datasets \\
    Randomly initialize parameters $\theta$ of learner $h$
    \LOOP 
    \STATE Sample minibatch $X_\mathrm{tr}, X_\mathrm{te}, Y_\mathrm{tr}$
    \STATE $
    \calL_\textrm{reg} \leftarrow \Lemp(h) 
        + \lambda \cdot \mathrm{NC}(X_\mathrm{tr}, X_\mathrm{te}, Y_\mathrm{tr}, h(X_\mathrm{tr}), h(X_\mathrm{te}))$
        \hfill \gls{nc}-regularized task loss \eqref{eq:task_SGD}
    \STATE $\theta \leftarrow \theta - \nabla_\theta \calL_\textrm{reg}$
        \hfill Gradient step
    \ENDLOOP
    \STATE $\gap(h) \leftarrow \Ltrue(h) - \Lemp(h)$
        \hfill Compute gap \eqref{eq:gap_def}
    \RETURN Snapshot $H = (X_\mathrm{tr}, X_\mathrm{te}, Y_\mathrm{tr}, h(X_\mathrm{tr}), h(X_\mathrm{te}), \gap(h))$
        \hfill Save to memory bank
\end{algorithmic}
\end{algorithm}%
\begin{algorithm}
\caption{Meta-Learning}
\label{alg:meta_learning_nc}
\begin{algorithmic}
    \REQUIRE Memory bank
    \STATE Randomly initialize parameters $\phi$ of $\mathrm{NC}$
    \WHILE{not converged}
    \STATE Sample $X_\mathrm{tr}, X_\mathrm{te}, Y_\mathrm{tr}, h(X_\mathrm{tr}), h(X_\mathrm{te}), G_{T, S}(h)$ from memory bank
    \STATE $\Delta \leftarrow G_{T, S}(h) - \mathrm{NC}(X_\mathrm{tr}, X_\mathrm{te}, Y_\mathrm{tr}, h(X_\mathrm{tr}), h(X_\mathrm{te}))$
    \STATE $\phi \leftarrow \phi - \nabla_\phi \calL_\mathrm{NC}(\Delta)$
        \hfill NC's loss function \eqref{eq:nc_loss}
    \ENDWHILE
\end{algorithmic}
\end{algorithm}
\subsection{Motivation: From Gap Estimate to Generalization Bound}
\label{subsec:bound}
We motivate our meta-learning objective through a simple method of extending the identity $\Lemp + \gap = \Ltrue$ to a probabilistic bound of $\Ltrue$ using any estimator of the gap $\gap$.

\begin{proposition}
\label{thm:nc_gen}
Let $D_\calH$ be a distribution of hypotheses,
and let $f: \calZ^m \times \calH \rightarrow \Real$ be any function of the training set and hypothesis.
Let $D_\Delta$ denote the distribution of $\gap(h) - f(S, h)$ where $h \sim D_\calH$,
and let $\Delta_1, \ldots, \Delta_n$ be i.i.d. copies of $D_\Delta$.
The following holds for all $\epsilon > 0$:
\begin{align} \label{eq:NC_thm}
  \Prob \left[ 
  \left| \Ltrue(h) - \Lemp(h) \right|  \leq f(S, h) + \epsilon
  \right]
  \geq 1 - \frac{ \left| \{i | \Delta_i > \epsilon \} \right| }{n} 
      - 2\sqrt{\frac{\log \frac{2}{\delta}}{2n}}.
\end{align}
\end{proposition}

We defer the proof to the supplementary material.
First, note that the role of $f$ in this bound mirrors that of complexity measures in previous generalization bounds.
Since we can compute $f$ given $S$ and $h$, we can restate \cref{thm:nc_gen} as stating that the regularized loss $\Lemp(h) + f(S, h)$ differs from $\Ltrue$ by at most $\epsilon$ (with the given probability).
Furthermore, making $f$ more accurately predict $\gap$ tightens this bound by decreasing the 
$\frac{ \left| \{i | \Delta_i > \epsilon \} \right| }{n}$
term.

Taking motivation from this result, our \gls{nc} meta-learns such a function $f$ by regressing towards the gap $\gap$.
Rather than designing a measure of complexity that yields a tight generalization bound, \gls{nc} directly learns such a measure in a data-driven way by posing the tightening of the bound as an optimization problem.

\subsection{Training}
\label{subsec:train}
We first illustrate \gls{nc}'s training loop.
Recall from \cref{sec:formulation} that we consider a meta-learning setup consisting of a set of related but different tasks.
Given a task $T$ with dataset $S$, the task learner minimizes the following regularized loss using stochastic gradient descent:
\begin{align}
\label{eq:task_SGD}
\calL_\mathrm{reg}(h) = \Lemp(h) + \lambda \cdot \mathrm{NC}(h).
\end{align}
We set $\lambda=0$ at initialization, and use a linear schedule where $\lambda=1$ after a certain number of episodes.
We calculate the two terms in $\calL_\mathrm{reg}(h)$ using minibatches, just as in regular supervised learning with neural networks.
Further note that one can use \gls{nc} regularization alongside any of the usual tricks for training, such as data augmentation or batch normalization.

The objective of \gls{nc} is to estimate the difference between $\Ltrue$ and $\Lemp$ for any of the hypotheses $h_0^T, \ldots, h_N^T$, for any task $T \sim \tau$.
\gls{nc} is a permutation-invariant neural network that takes features of the function $h$ and minibatches of the data as input and outputs a scalar.
We train \gls{nc} using the Huber loss \citep{huber1992robust} with target $\gap$:
\begin{align}
\label{eq:nc_loss}
\calL_\mathrm{NC}(\Delta)=
\left\{\begin{array}{ll}
    \frac{1}{2} \Delta^{2} 
        & \text{for } \Delta \leq 1 \\
    |\Delta|-\frac{1}{2}
        & \text{otherwise}
\end{array}\right.,
\end{align}
where $\Delta = \gap(h) - \mathrm{NC}(h)$.
We found that the Huber loss was more stable than the standard MSE loss, likely because the scale of $\gap$ can vary widely depending on $h$.

\subsection{Architecture}
\label{subsec:arch}

We now describe the architecture of \gls{nc} used in our experiments.
We have mentioned in \cref{thm:nc_gen,eq:task_SGD} that \gls{nc} takes a representation of the function $h$ as input.
We accomplish this by passing both data $x$ and predictions $h(x)$ to \gls{nc}.
In terms of the function $h$, the tuple $(x, h(x))$ can be written as $\delta_x[(\mathrm{id}, h)]$ where $\delta_x$ is the evaluation functional at $x$ and $\mathrm{id}$ is the identity function.
This representation of $h$ captures the behavior of $h$ at the datapoints considered during training and evaluation.
In our experiments, this structure sufficed for extracting relevant features of $h$.

\paragraph{Regression}
We first describe \gls{nc} when each task $T$ is a regression task with vector data ($x \in \Real^D$).
Let 
$X_\mathrm{tr} \in \Real^{m \times D}$, 
$X_\mathrm{te} \in \Real^{m' \times D}$, 
$Y_\mathrm{tr} \in \Real^{m \times 1}$ 
denote train data, test data, and train labels, respectively.
The learner's hypothesis $h$ produces outputs 
$h(X_\mathrm{tr}) \in \Real^{m \times 1}$ and 
$h(X_\mathrm{te}) \in \Real^{m' \times 1}$ 
for train and test data.
\gls{nc} first embeds all data with a shared encoding network $f_\mathrm{enc}$, which is an MLP that operates row-wise on these matrices:
\begin{align}
    f_\mathrm{enc}(X_\mathrm{tr}) = e_\mathrm{tr} \in \Real^{m \times d}, \quad
    f_\mathrm{enc}(X_\mathrm{te}) = e_\mathrm{te} \in \Real^{m' \times d}.
\end{align}
These embeddings are fed into a multi-head attention layer \citep{vaswani2017attention} where queries, keys, and values are 
$Q=e_\mathrm{te}$, 
$K=e_\mathrm{tr}$,
$V=[e_\mathrm{tr}, y_\mathrm{tr}] (\in \Real^{m \times (d+1)})$,
respectively.
The output of this attention layer is a set of $m'$ items, each corresponding to a test datapoint:
\begin{align}
    \label{eq:MHA}
    f_\mathrm{att}(Q, K, V)
    = e_\mathrm{att} \in \Real^{m' \times d}.
\end{align}
Finally, these embeddings are passed through a decoding MLP network and averaged:
\begin{align}
    \mathrm{NC}(X_\mathrm{tr}, X_\mathrm{te}, Y_\mathrm{tr}, h(X_\mathrm{tr}), h(X_\mathrm{te}))
    = \frac{1}{m'} \sum_{i=1}^{m'} f_\mathrm{dec}(e_\mathrm{att})_i \in \Real.
\end{align}
Note that \gls{nc} is permutation invariant because it is consists of permutation invariant components.
This property is essential since \gls{nc}'s objective is also invariant with respect to permutation of the input dataset.

\paragraph{Classification}
\begin{figure}
    \centering
    \includegraphics[width=\linewidth]{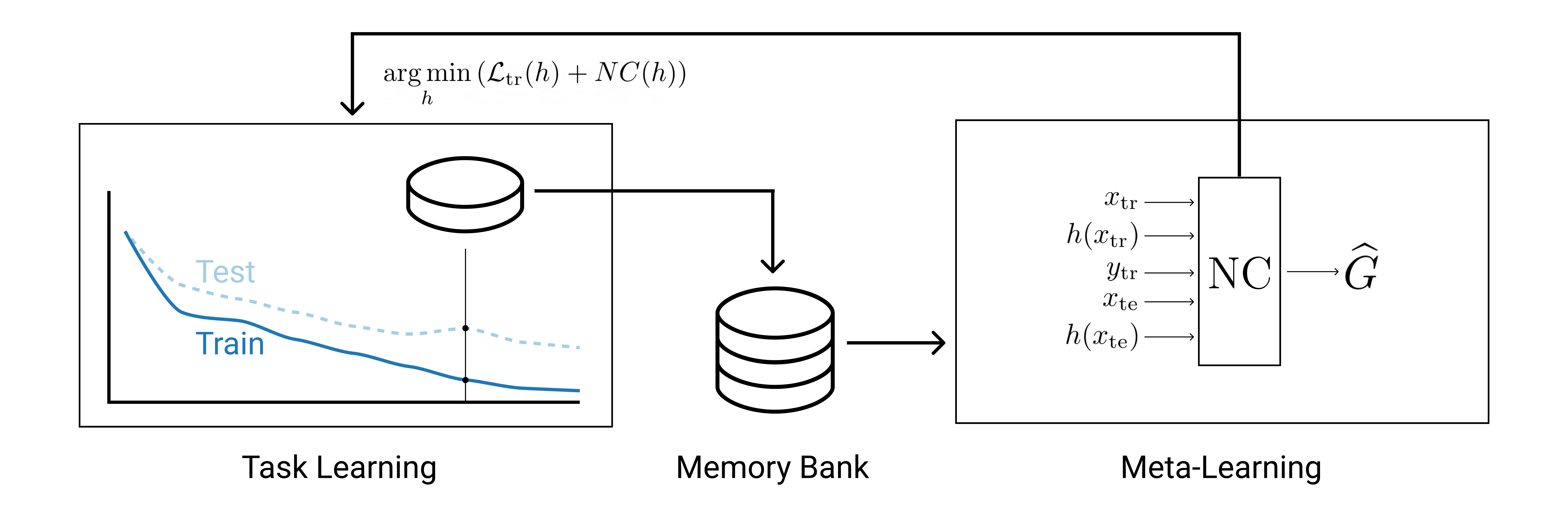}
    \caption{
        \gls{nc} regularizes task learning.
        We store snapshots of task learning in a \textit{memory bank},
        from which we uniformly sample batches to train \gls{nc}.
        }
    \label{fig:concept}
\end{figure}
The architecture of \gls{nc} for classification tasks is identical to that of regression, except for the following additional interaction layer used to compute $V$.
Representing labels as one-hot vectors in a classification task with $c$ classes gives $Y_\mathrm{tr} \in \Real^{m' \times c}$.
Instead of concatenating $e_\mathrm{tr}$ and $Y_\mathrm{tr}$ as in \eqref{eq:MHA},
we use a bilinear layer to produce $V$:
\begin{align}
    \label{eq:bilin}
    V = \mathbb{W}(e_\mathrm{tr}, [Y_\mathrm{tr}, \mathbf{1}, \calL(X_\mathrm{tr})])) \in \Real^{m' \times d} \quad
    (\mathbb{W} \in \Real^{d \times d \times (c+2)}).
\end{align}
Note that we concatenate a vector of ones and the train loss to $Y_\mathrm{tr}$ before passing into the bilinear layer:
this vector acts like a residual connection for the embedding $e_\mathrm{tr}$, allowing its information to freely flow to the next layer.
See \cref{sec:experiments} for an ablation study on each of our architectural choices.
The bilinear layer (\cref{eq:bilin}) generalizes the interaction layer proposed in \cite{xu2019metafun}: while they explicitly choose a subnetwork to use according to class, \eqref{eq:bilin} implicitly multiplies $0$ to all but one of the $c$ weights in each slice of the last dimension.
Additionally, to scale \gls{nc} up to high-dimensional image data such as the CIFAR dataset, we use a convolutional neural network for the encoder $f_\mathrm{enc}$.

Because training runs \eqref{eq:task_SGD} are time-consuming for large networks $h$, we use a \textit{memory bank} to store and re-use the information necessary for the meta-learning loss \eqref{eq:nc_loss}.
Specifically, we store tuples 
$(X_\mathrm{tr}, X_\mathrm{te}, Y_\mathrm{tr}, h(X_\mathrm{tr}), h(X_\mathrm{te}))$ 
along with the observed gap $\Lemp(h) - \Ltrue(h)$.
This memory bank has manageable memory cost because we can store only the indices for $X_\mathrm{tr}, X_\mathrm{te}$, and the other tensors have low dimensions.
We randomly sample minibatches of such tuples to train \gls{nc} with the meta-learning loss \eqref{eq:nc_loss}.
\cref{fig:concept} shows how the memory bank interacts with \gls{nc}.

\subsection{Interpretations}
\label{subsec:interpretations}
We provide several different interpretations the \gls{nc} framework.

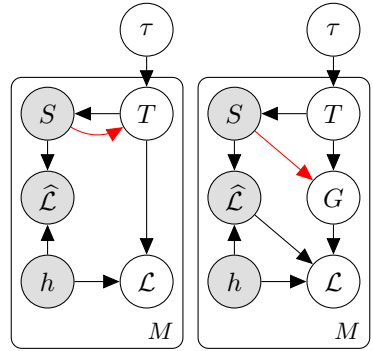
\begin{wrapfigure}{R}{0.35\textwidth} \centering
\begin{subfigure}[b]{.49\linewidth} \centering
\begin{tikzpicture}[x=.6cm,y=.4cm]
  \node[latent]                         
  (tau) {$\tau$};
  \node[latent, below=of tau]                         
  (T) {$T$};
  \node[obs, left=of T]
  (S) {$S$};
  \node[obs, below=of S]                         
  (Lemp) {$\widehat{\calL}$};
  \node[obs, below=of Lemp]                         
  (h) {$h$};
  \node[latent, right=of h]
  (Lgen) {$\calL$};

  \edge {tau} {T} ; %
  \edge {T} {S, Lgen} ; %
  \edge {S} {Lemp} ; %
  \edge {h} {Lemp, Lgen} ; %
  \draw [red, ->] (S) to [bend right] (T);

  \plate {} {(T) (S) (Lemp) (Lgen) (h)} {$M$} ;
\end{tikzpicture}
\end{subfigure}
\begin{subfigure}[b]{.49\linewidth}
\centering
\begin{tikzpicture}[x=.6cm,y=.4cm]
  \node[latent]                         
  (tau) {$\tau$};
  \node[latent, below=of tau]                         
  (T) {$T$};
  \node[latent, below=of T]                         
  (G) {$G$};
  \node[obs, left=of T]
  (S) {$S$};
  \node[obs, below=of S]                         
  (Lemp) {$\widehat{\calL}$};
  \node[latent, below=of G]
  (Lgen) {$\calL$};
  \node[obs, below=of Lemp]                         
  (h) {$h$};

  \edge {tau} {T} ; %
  \edge {T} {S, G} ; %
  \edge {G, Lemp} {Lgen} ; %
  \edge {S} {Lemp} ; %
  \edge {h} {Lemp, Lgen} ; %
  \draw [red, ->] (S) to (G);

  \plate {} {(T) (S) (Lemp) (Lgen) (h)} {$M$} ;
\end{tikzpicture}
\end{subfigure}
\caption{
  Graphical models corresponding to (left) Neural Processes and (right) \gls*{nc}.
  Observed nodes are shaded and red arrows denote amortized inference.
}
\label{fig:plate_vertical}
\vspace{-30pt}
\end{wrapfigure}

\paragraph{Causality in Generalization}
As noted in \cite{jiang2019fantastic}, the correlation between a complexity measure and generalization does not directly imply a causal relationship.
Since regularizing the complexity measure implicitly assumes that reducing the measure will cause the model to generalize, this lack of a causal connection can be problematic.
\gls{nc}'s framework provides an alternative way around this issue:
because we keep using \gls{nc} as a regularizer, it continually gets feedback on whether its predictions have caused the task learner to generalize.

\paragraph{Meta-learned Complexity Measure}
As mentioned in \cref{sec:intro}, many recent works attempt to understand generalization in deep networks by proposing novel complexity measures.
Such measures are designed to correlate well with generalization while being directly computable for any given set of parameters.
\gls{nc} can be seen as a meta-learned complexity measure, and its target \eqref{eq:gap_def} is the generalization gap.
Instead of hand-designing an appropriate complexity measure, \gls{nc} meta-learns it by regressing towards observed degrees of generalization.

\paragraph{Optimal Regularizer}
A standard approach to generalization is to augment the empirical loss $\Lemp$ by adding a regularization term $\lambda$: $\calL_\mathrm{reg} = \Lemp(h) + \lambda(h) $.
Since the purpose of $\lambda$ is to make the regularized loss $\calL_\mathrm{reg}$ close to the true loss $\Ltrue$, 
we argue that the \textit{optimal regularizer} for task $T$ is the function that makes $\calL_\mathrm{reg}(h) = \Ltrue(h)$ for all $h$.
This unique "optimal regularizer" is exactly $\gap$;
therefore, \gls{nc} can be seen as a learned approximation to this optimal regularizer.

\paragraph{Neural Processes and Sufficient Statistics of True Loss}
We contrast the graphical models of \gls{nc} with the 
\gls{np} \citep{garnelo2018neural} in \cref{fig:plate_vertical}.
Both approaches involve a single meta-learner which observes multiple tasks to achieve low test loss.
The two approaches infer different sufficient statistics for the true loss $\Ltrue$.
\gls{np} infers the data distribution of $T$ and \gls{nc} infers the gap $G$.
While both $G$ and the data distribution are sufficient for reconstructing $\Ltrue$, $G$ has much lower dimension:
$G(h) \in \Real$, whereas $T$ is a complex distribution over 
$\calX \times \calY = \Real^{d_X + d_Y}$.

\paragraph{Actor-Critic}
Generalizing to unseen test data can be seen as a reinforcement learning environment with known dynamics:
the observations are train data and train loss, and the selection of hypothesis $h$ is the action.
The objective is to maximize the return, which is $-\Ltrue$.
Within this interpretation, our approach is an actor-critic method where \gls{nc} takes the value network's role.

\section{Related Works}

\paragraph{Complexity Measures for Deep Networks}
The question of why deep networks generalize despite being over-parameterized has been the focus of many recent works.
Building on traditional generalization theory \citep{vapnik1999overview,mcallester1999pac}, such works have adapted 
PAC-Bayes bounds \citep{dziugaite2017computing,zhou2018non}
and norm-based bounds \citep{neyshabur2015norm,bartlett2017spectrally} to deep networks.
Other works have proposed measures that empirically correlate with generalization \citep{keskar2016large,liang2017fisher,DBLP:journals/corr/abs-1901-01672}.
This work proposes an alternative approach to the problem of explaining generalization. 
While these previous works rely on human-designed measures of complexity, \gls{nc} learns such a measure in a data-driven way, allowing it to learn to reflect the complex interaction between the function $h$ and the data distribution.

\paragraph{Predicting Generalization}
A few recent works have proposed to predict generalization using function approximation.
\citep{jiang2018predicting} learns linear regression coefficients to predict the generalization gap, and \cite{yak2019towards} extends their model by using a small neural network instead of a linear model.
\cite{unterthiner2020predicting} proposes to predict neural network accuracy from various weight features.
These approaches are all \textit{correlational} analyses of generalization.
As discussed in \cref{subsec:interpretations}, \gls{nc}'s framework allows for the discovery of causal factors of generalization because of the back-and-forth interaction between task learning and predicting generalization.

\paragraph{Meta-Learning} \nocite{Lee2018}
Our method falls within the framework of meta-learning \citep{thrun1998learning,schmidhuber1996simple},
in which a model learns useful information about the learning process itself through interactions with a set of different but related tasks.
Recent methods formulate the meta-learning problem as learning 
optimizers \citep{ravi2016optimization},
data embeddings \citep{snell2017prototypical},
initial parameters \citep{finn2017model}, 
or parameter priors \citep{kim2018bayesian}.
A key difference is that \gls{nc} is learner-agnostic: we can use an \gls{nc} model trained on one class of task learners to regularize other task learners (e.g., different architecture, activation, optimizer).
Additionally, using \gls{nc}'s output as a regularization loss makes it more stable in long training runs than previous meta-learning algorithms.

MetaReg \citep{balaji2018metareg} proposes to meta-learn a weighted $L_1$ regularizer. 
It is probably the most similar method because they also learn a regularizer in a meta-learning setup.
However, their regularizer does not transfer to different network architectures because it operates in parameter space.
In contrast, \gls{nc} learns a complexity measure in function space, allowing it to generalize to different architectures and randomly initialized networks.
\section{Experiments}
\label{sec:experiments}

\subsection{Sinusoid Regression}
\label{subsec:sine_regression}

\begin{figure} \centering
\begin{minipage}{.27\textwidth} \centering 
\vspace{-25pt}
    \includegraphics[width=\linewidth]{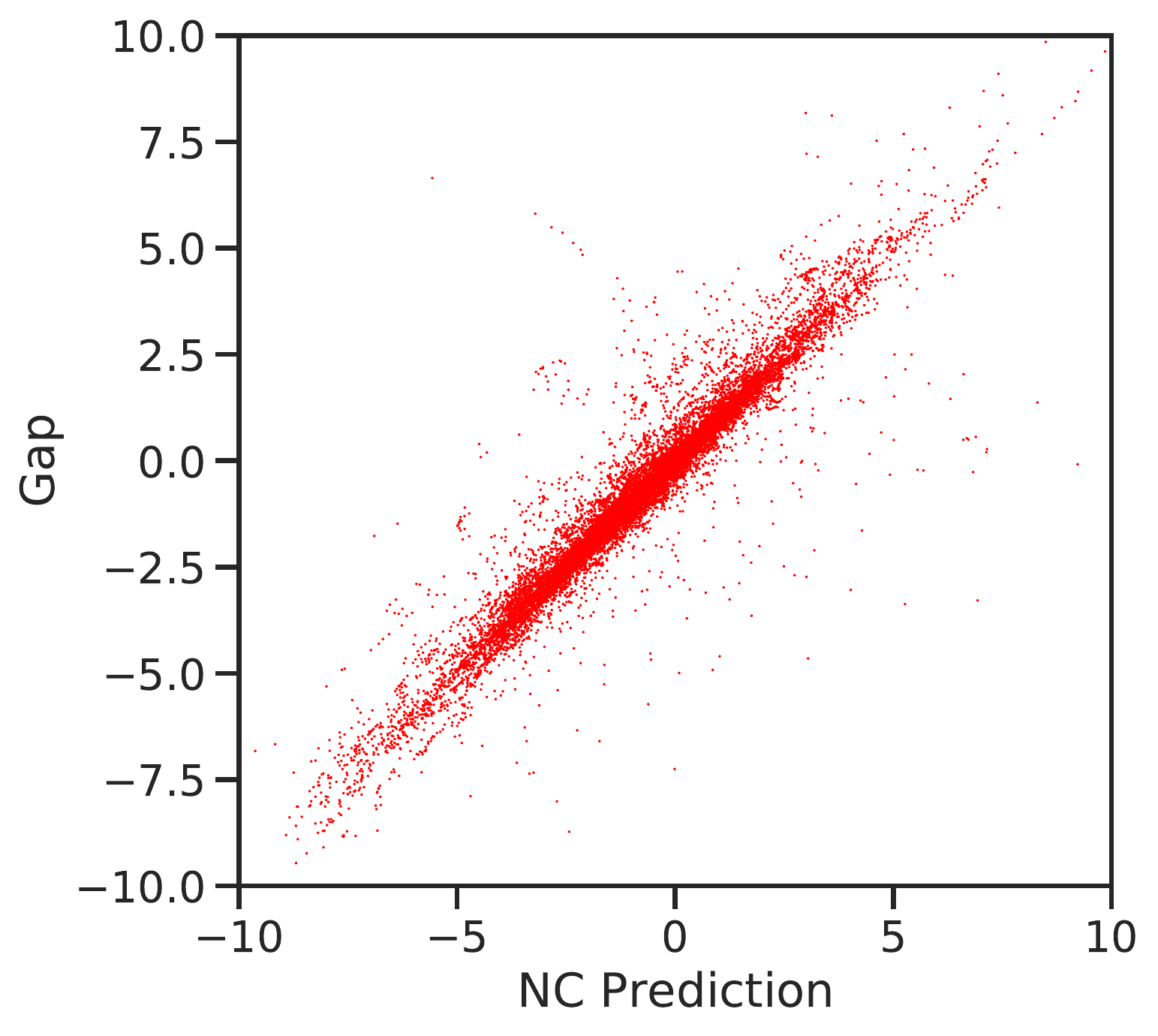}
\caption{ \gls{nc} predictions and true gap values. } 
\label{fig:R} \label{fig:nc_R}
\end{minipage} \hfill
\begin{minipage}{.31\textwidth}
\vspace{-30pt}
\centering \includegraphics[width=\linewidth]{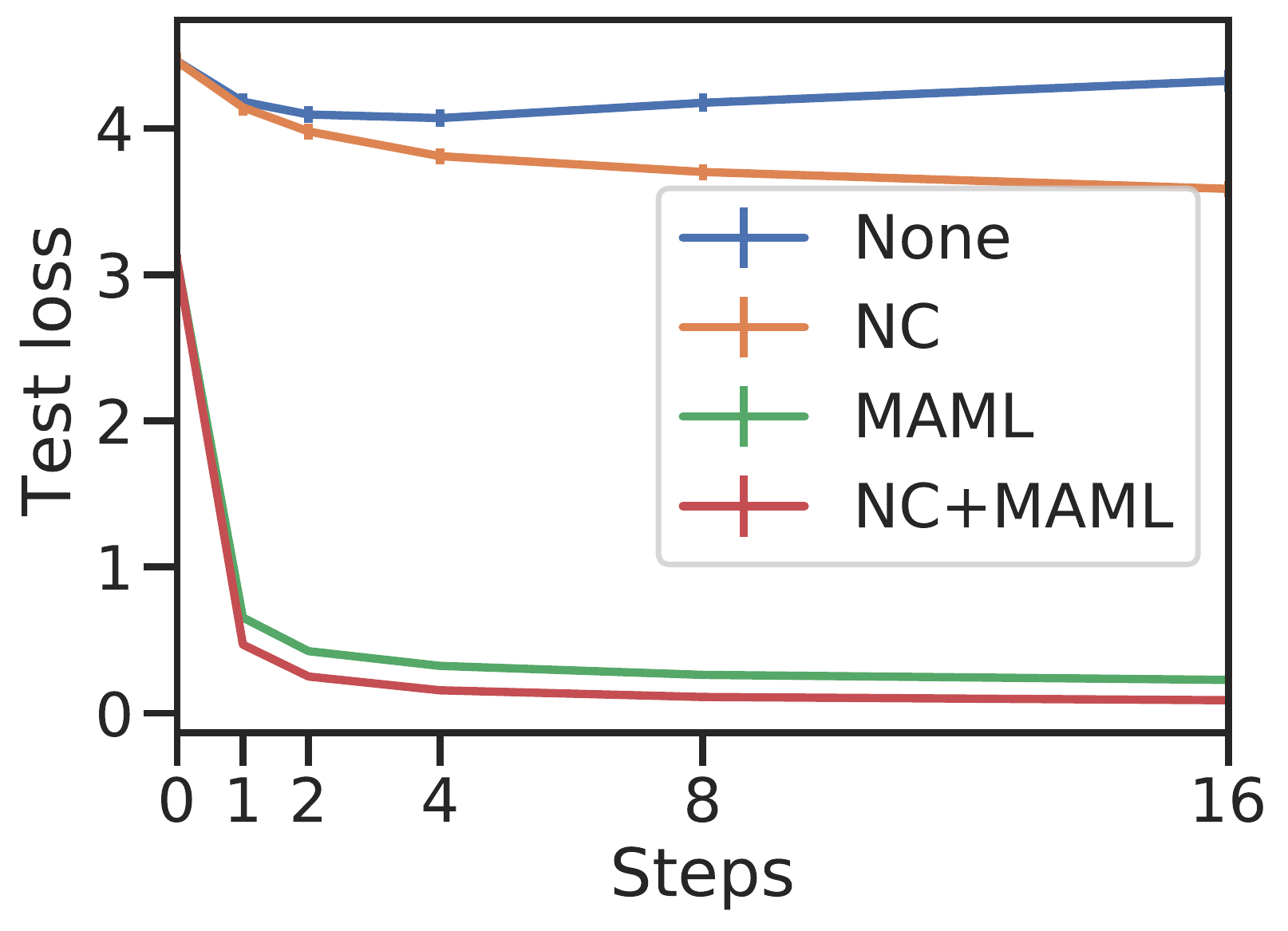}
\caption{Test loss of combinations of \gls{nc} and MAML.} 
\label{fig:nc_maml}
\end{minipage} \hfill
\begin{minipage}{.33\textwidth}
\vspace{-20pt}
\begin{tabular}{lc}
\toprule%
Method & Accuracy \\
\midrule%
Full & \BF 92.93 \\
\midrule%
- Huber loss & 89.74 \\
- Bias & 84.30 \\
- Loss conditioning & 72.09 \\
- Bilinear layer & 23.69 \\
\bottomrule%
\end{tabular}
\caption{Ablation study; we sequentially removed each architectural component.}
\label{tab:class_ablation_1shot}
\end{minipage}
\vspace{-10pt}
\end{figure}

To illustrate the basic properties of \gls{nc}, we begin with a toy sinusoid regression problem introduced in \cite{finn2017model}.
Each task is a sine function $x \mapsto A \textrm{sin} (x+b)$
where $x, A, b$ are uniformly sampled from $[-5, 5]$, $[0.1, 5]$ and $[0, \pi]$, respectively.
We consider $10$-shot learning and use the mean squared error as the loss function.
We measure test losses using a held-out test set of $15$ datapoints.
During meta-training, the layer size, activation, number of layers, learning rate, and number of steps were all fixed to ($40$, ReLU, $2$, $0.01$, $16$), respectively.

\paragraph{NC-Gap Fit}
In \cref{fig:nc_R}, we compare the predictions of a trained \gls{nc} model with the generalization gap $\gap$ for many batches.
These two values are strongly correlated ($R^2=0.9589$), indicating that \gls{nc} is indeed capable of predicting the gap based on the complexity of the learner's hypothesis $h$.

\paragraph{VS Other Regularizers}
We compared \gls{nc} against various other regularization methods.
We report these results in the appendix due to space issues.
\gls{nc} performs all other methods by a large margin because learners tend to overfit very quickly in this few-shot regression problem.

\paragraph{Integration with MAML}
We investigated whether \gls{nc} can integrate with MAML~\citep{finn2017model}, an alternative meta-learning approach.
We first note that these two methods solve two very different problems: \gls{nc} aims to regularize any randomly initialized network while MAML simply finds one set of initial parameters from which learning occurs quickly.
We first trained a MAML model and then trained \gls{nc} using snapshots obtained from MAML trajectories.
\Cref{fig:nc_maml} shows that \gls{nc} successfully reduces the final test loss for both settings: with and without MAML initializations.
These results indicate that the regularization effect of \gls{nc} is orthogonal to that of MAML, and that future improvements in either direction can benefit the other.

\paragraph{Learning Curve}
We show the train and test loss curves of a task learner regularized by \gls{nc} in \cref{fig:nc_curves}.
The test loss is lower than the train loss throughout training, which is a trend that we observed in all experimental settings we considered.
In other words, the estimate of \gls{nc} is a precise enough surrogate for $\gap$ that minimizing it results in negative $\gap$.

\subsection{Out-of-distribution Task Learners}
\begin{figure}[t] \centering
\begin{minipage}{0.3\linewidth} \centering 
    \includegraphics[width=\linewidth]{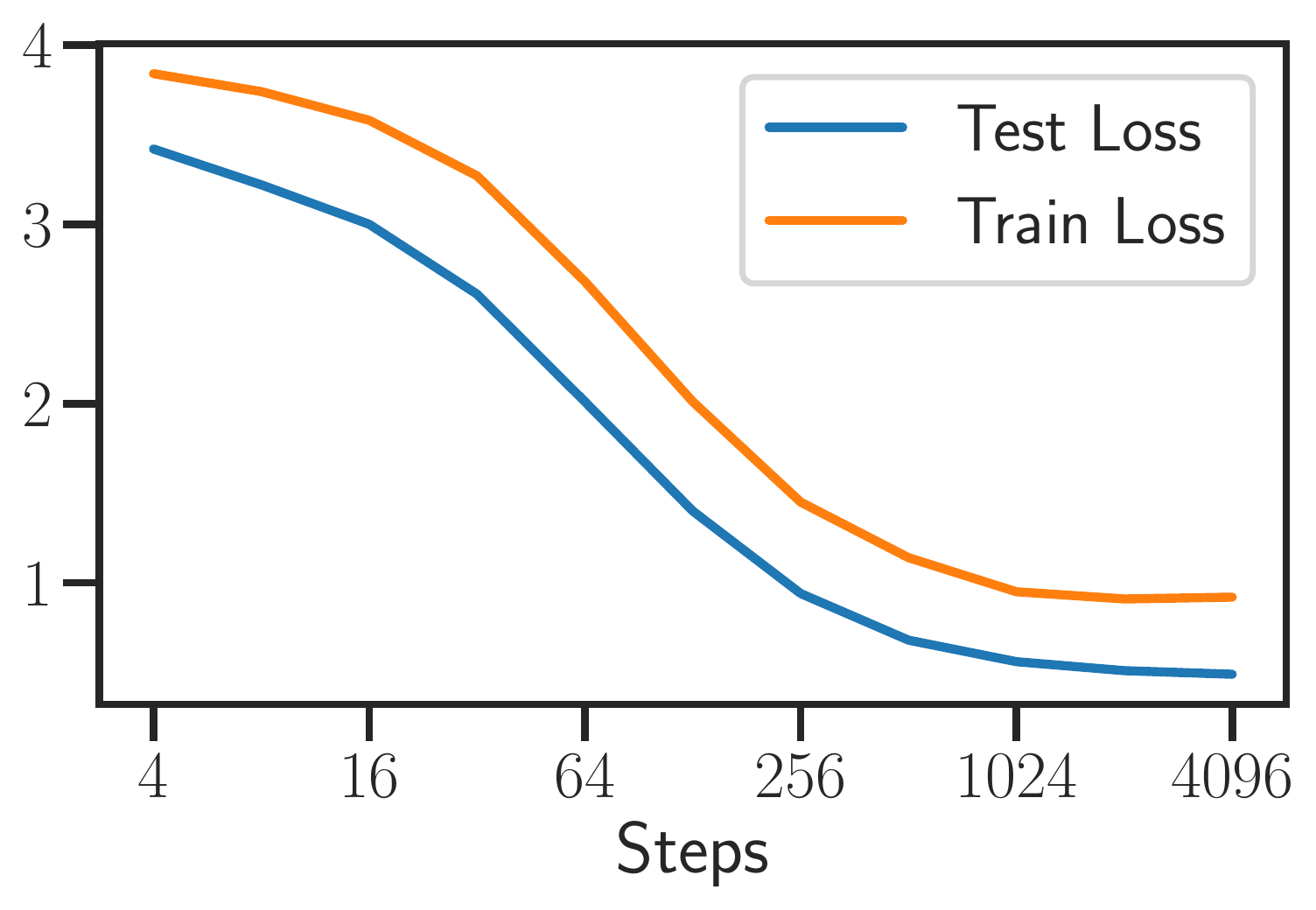}
    \caption{Train and test loss curves of \gls{nc} regularization.} \label{fig:nc_curves}
\end{minipage}%
\hfill%
\begin{minipage}{.33\linewidth} \centering
    \includegraphics[width=0.9\linewidth]{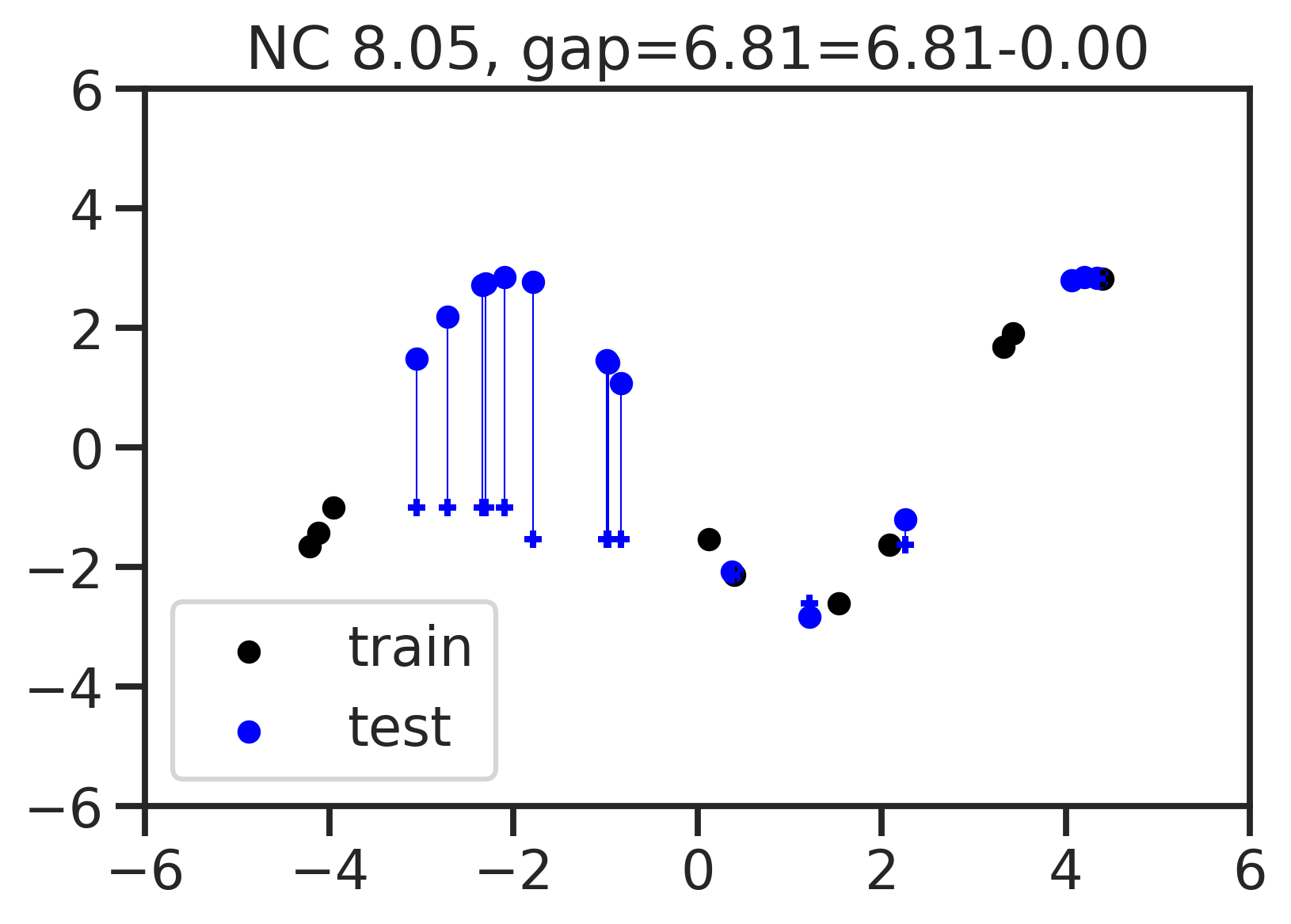}
\vspace{-7pt}
\caption{ Nearest neighbor regression. 
Circles and plus signs represent targets and predictions. }
\label{fig:reg_vis}
\end{minipage} \hfill
\begin{minipage}{.33\linewidth} \centering
    \begin{tabular}{lcc}
    \toprule
        & SVHN & CIFAR \\
    \midrule
    Baseline
        & 93.23 & 79.76 \\
    \midrule
    size $\times 2$
        & 93.59 & 79.64 \\
    $\mathrm{NC}$ (ours)
        & \BF 93.83 & \BF 81.15 \\
    \midrule
    size $\times 4$
        & \color{blue} 93.88 & 80.47 \\
    \bottomrule
    \end{tabular}
    \label{tab:capacity}
    \caption{Comparison to networks with more capacity.} 
\end{minipage}
\includegraphics[width=\linewidth]{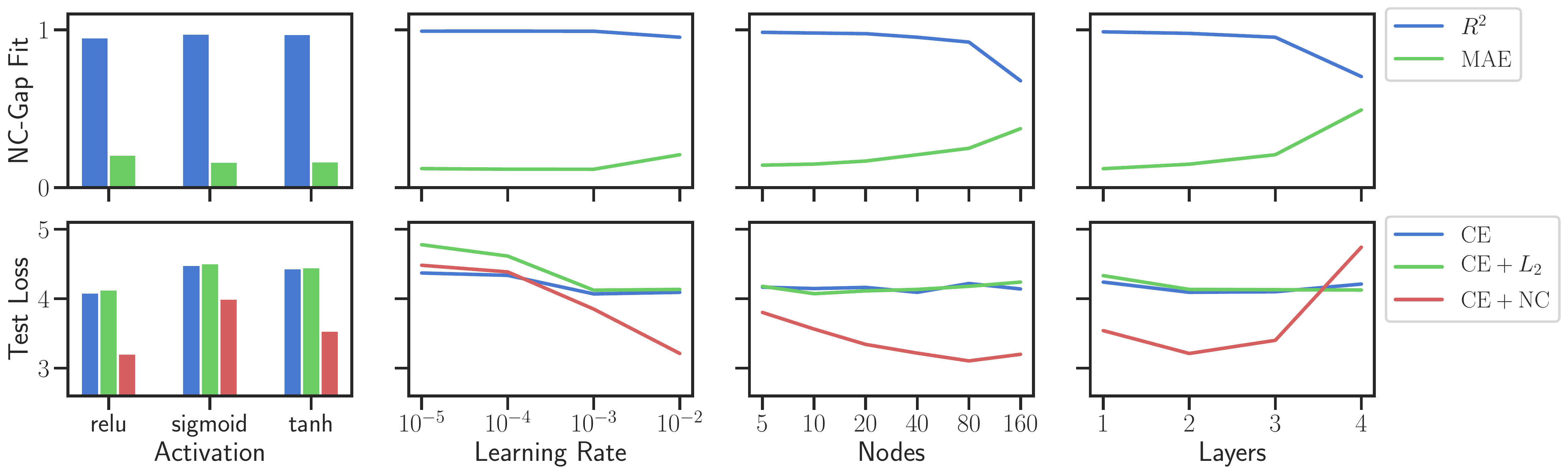}
\caption{
    Evaluation of Out-of-distribution task learners.
    The x axis shows the altered learner hyperparameter.
    (top) NC-gap fit statistics and
    (bottom) test loss after learning with \gls{nc} and baselines.
    } \label{fig:ood}
\end{figure}

\paragraph{Visualization of Simple Learners}
We observed the behavior of \gls{nc} when given hypotheses from closed-form learners with very distinct properties.
We consider a nearest neighbor learner.
We show $(X_\mathrm{tr}, X_\mathrm{te}, Y_\mathrm{tr}, Y_\mathrm{te}, h(X_\mathrm{tr}), h(X_\mathrm{te}))$ and gap values along with \gls{nc} predictions in \cref{fig:reg_vis}.
Note that while we show test targets $Y_\mathrm{te}$ in the figure, \gls{nc} does not observe them.
Even though the shown predictive function $h$ perfectly fits the training data, \gls{nc} penalizes the function because it does not have a sinusoidal shape.

\paragraph{Changing NN Learner Hyperparameters}
We evaluated how well \gls{nc} can generalize to other task learners on the sinusoid regression task in \cref{subsec:sine_regression}.
We measured performance while alternating four different learning algorithm hyperparameters: activation, learning rate, nodes per layer, and the number of layers.
We used the same \gls{nc} model which was trained using only ($\mathrm{relu}$, $10^{-2}$, $40$, $2$) for these hyperparameters.
In \cref{fig:ood}, we measure how well \gls{nc} fits the $\gap$ through their $R^2$ statistic and their mean absolute error (MAE).
\gls{nc}'s predictions are accurate even when the learners are changed, only degrading when the learner is significantly more expressive ($160$ nodes and $4$ layers).
\cref{fig:ood} additionally reports the test losses of \gls{nc} regularization compared to the cross-entropy loss and $L_2$ regularization.
\gls{nc} regularization shows consistent improvements except for when the network architecture was too different ($4$ layers).
This experiment demonstrates that \gls{nc}'s complexity measure captures the properties of $h$ itself, regardless of its specific parameterization.
We emphasize that such a transfer between different task learners is not possible with other meta-learning approaches. \citep{finn2017model,snell2017prototypical,garnelo2018neural}.

\subsection{Few-shot Image Classification}
\begin{figure} \centering
\vspace{-10pt}
\begin{minipage}{0.32\textwidth} \centering
    \includegraphics[width=.98\linewidth]{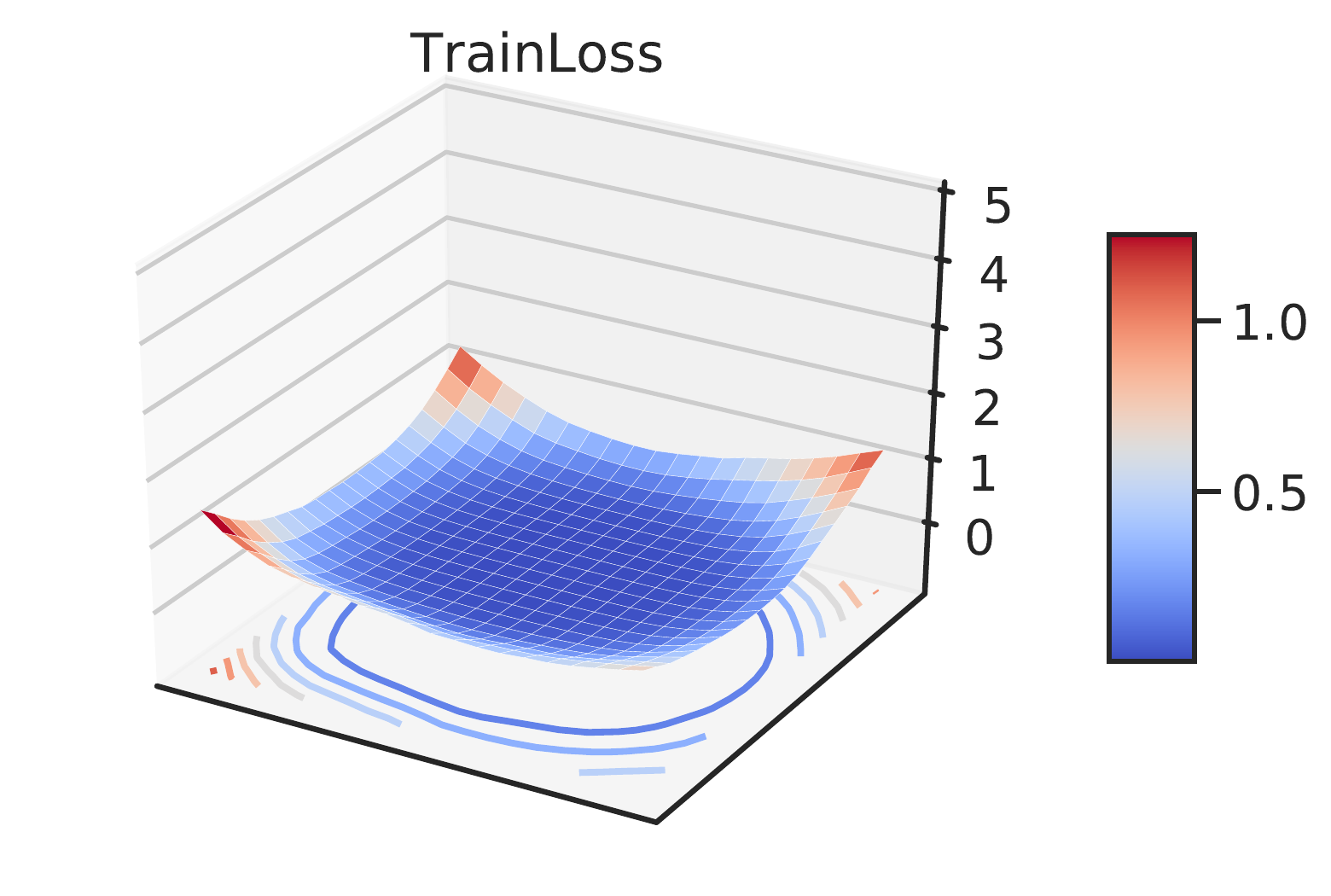} \\
\end{minipage}%
\begin{minipage}{0.32\textwidth} \centering
    \includegraphics[width=.98\linewidth]{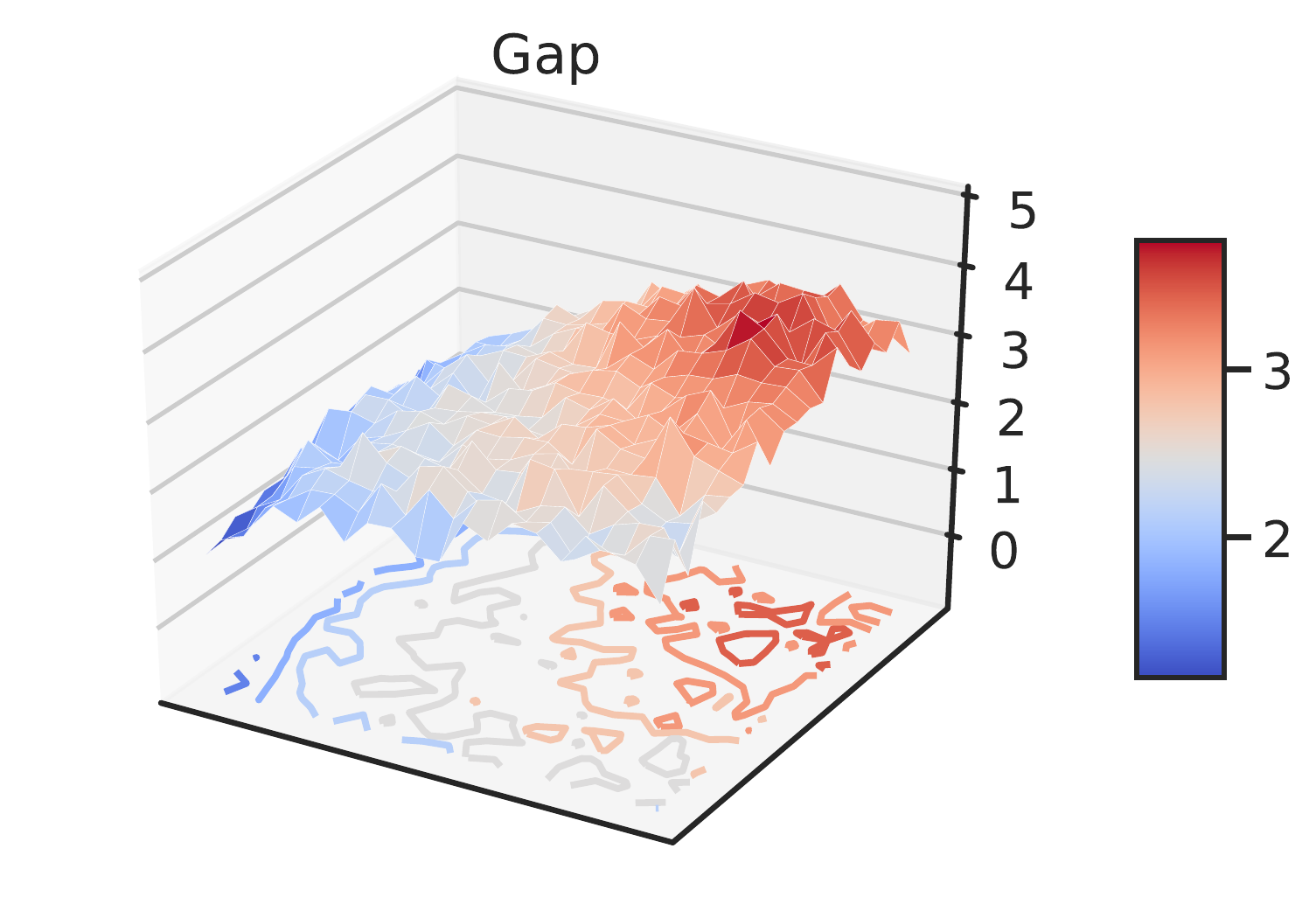}
\end{minipage}%
\begin{minipage}{0.32\textwidth} \centering
    \includegraphics[width=.98\linewidth]{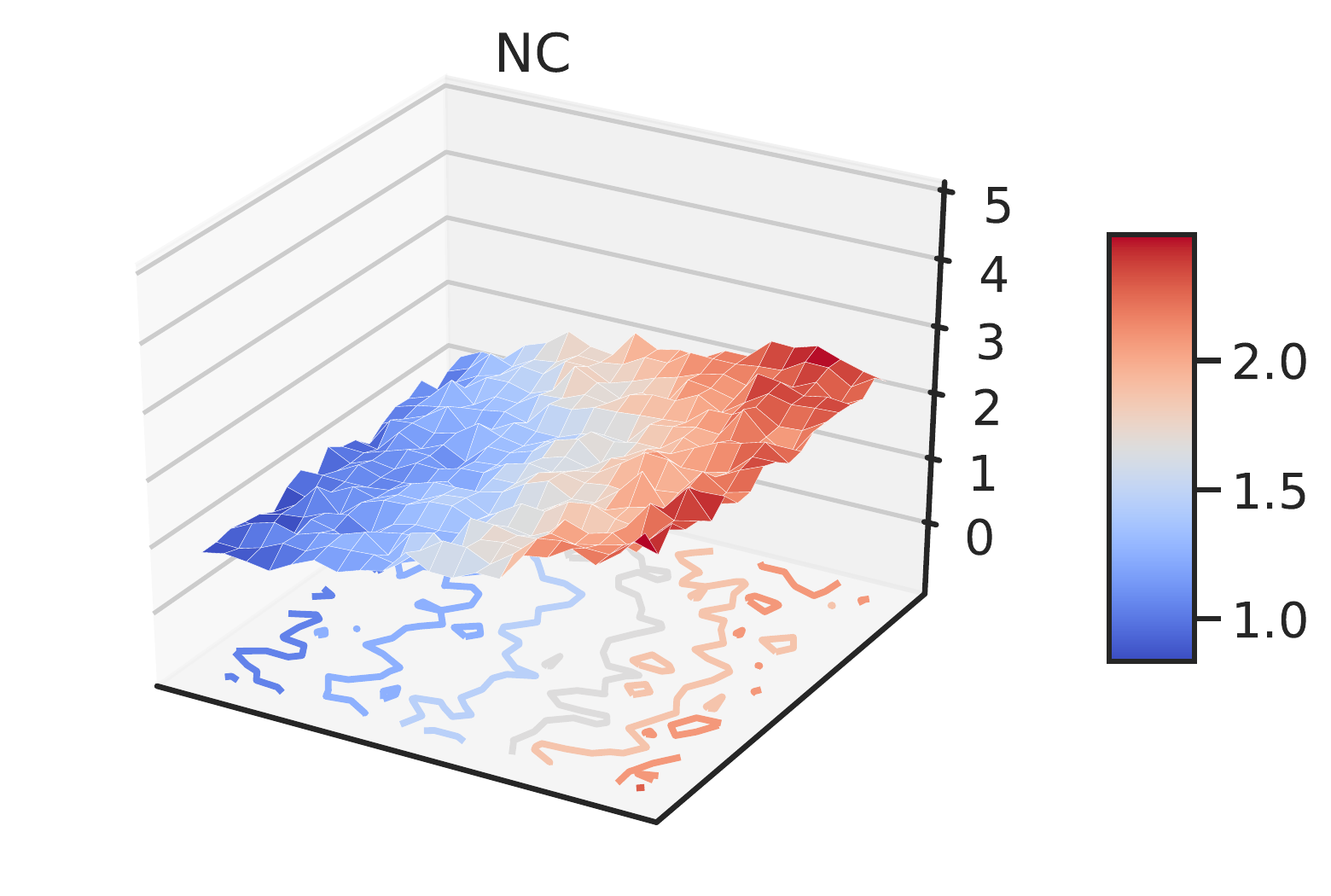}
\end{minipage}%
\caption{Visualization of loss surfaces. Best viewed zoomed in.}
\label{fig:surface_vis}
\vspace{-10pt}
\end{figure}
\paragraph{Ablation Study}
To validate our architectural choices for \gls{nc}, we performed an ablation experiment using a $10$-way $1$-shot classification task on the Omniglot \cite{lake2015human} dataset.
Results in \cref{tab:class_ablation_1shot} show the performance of several architectures for \gls{nc}.
First note that removing Huber loss and using MSE loss degrades performance, likely due to large gradients when the difference between \gls{nc}'s prediction and $\gap$ is large.
Furthermore, removing any of the additional components for the classification model (bias, train loss, bilinear layer) reduces accuracy, with the bilinear layer being the most critical for performance.

\paragraph{Loss Surface Visualization}
We use the filter-wise normalization technique introduced in \cite{li2018visualizing} to visualize the loss surfaces of a learner at convergence.
\cref{fig:surface_vis} shows that the train loss is at a stable local minimum, but the generalization gap can decrease further by moving in a specific direction.
Because \gls{nc} correctly captures the trend of the gap, minimizing the \gls{nc}-regularized loss would move the learner to a region with lower test loss.

\subsection{Single Image Classification Tasks}

\begin{figure}
    \includegraphics[width=0.33\linewidth]{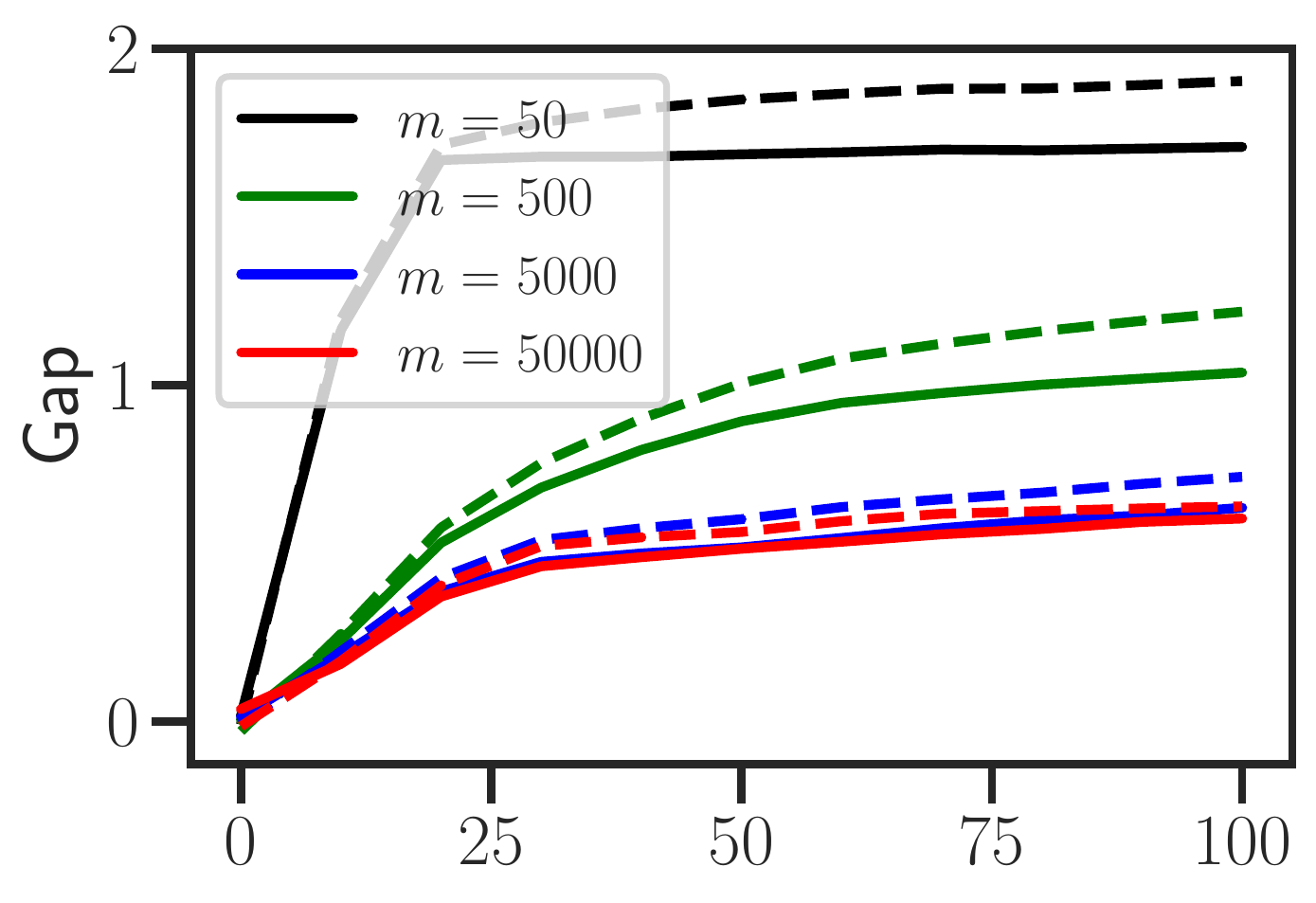}
    \includegraphics[width=0.33\linewidth]{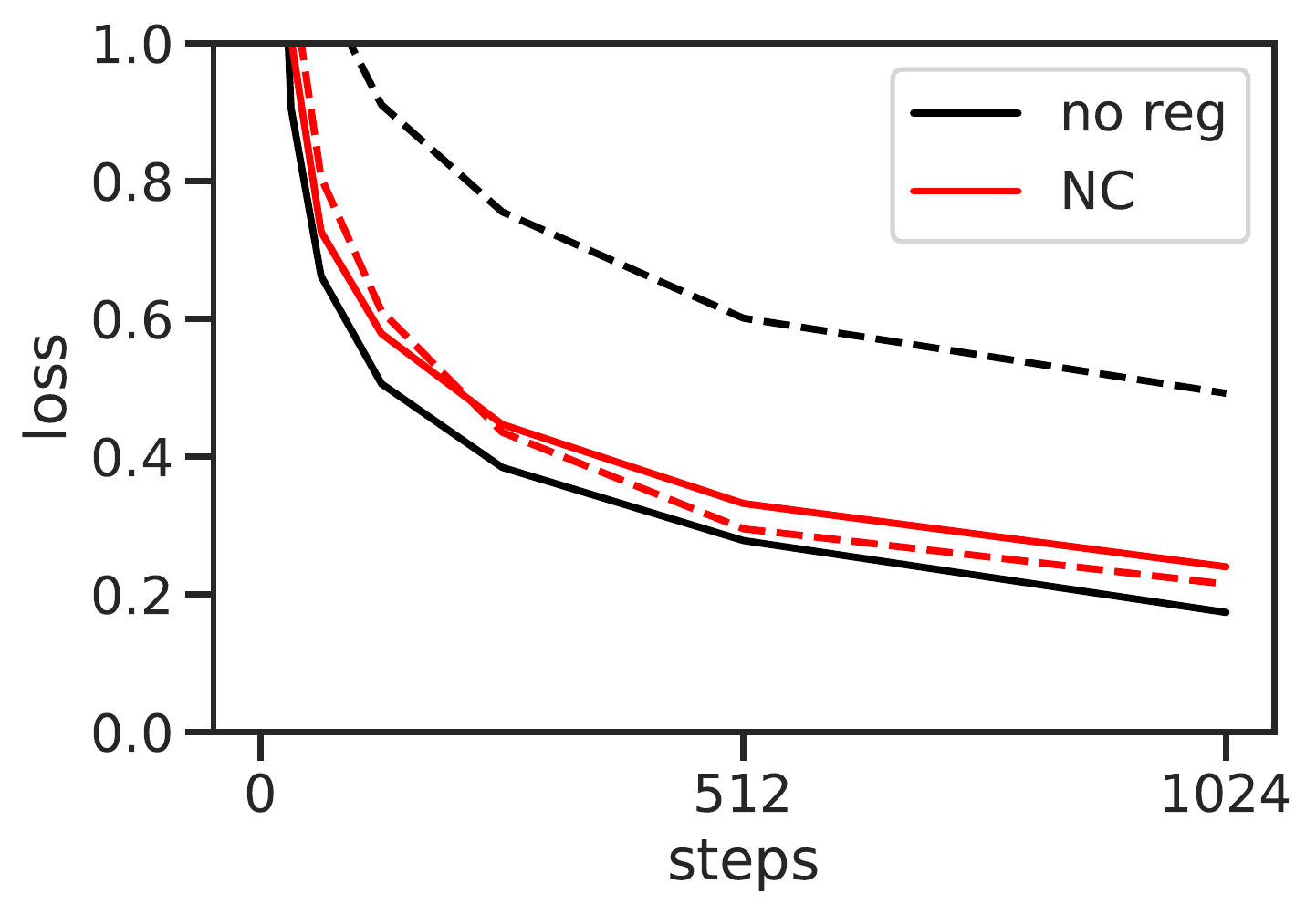}
    \includegraphics[width=0.33\linewidth]{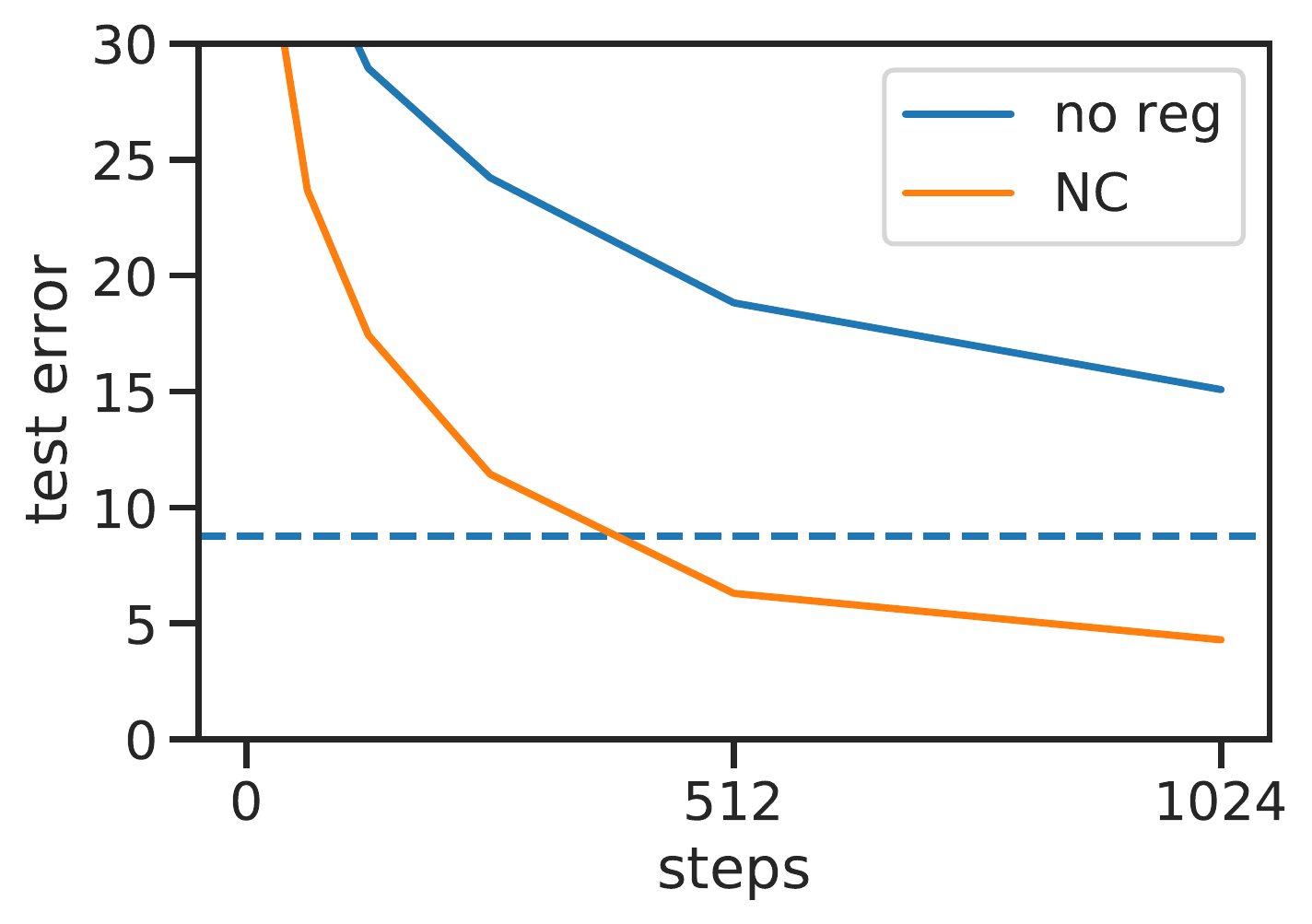}
\caption{
    Single-task regularization on the KMNIST dataset.
    Left: \gls{nc} estimates (solid lines) and gap values (dashed lines) when training with datasets of different size.
    Center: Learning curve of train (solid lines) and test (dashed lines) losses with and without \gls{nc}.
    Right: Learning curve of test error with and without \gls{nc}. The dashed horizontal line represents baseline performance after convergence.
}
\label{fig:MNIST_exps}
\end{figure}

\begin{table}[t] \centering
\begin{tabular}{lccccc}
\toprule%
    & MNIST & FMNIST & KMNIST & SVHN & CIFAR-10 \\
\midrule%
Cross-Entropy 
    & \msd{98.31}{0.12} & \msd{88.21}{0.05} & \msd{91.12}{0.08} & \msd{93.23}{0.44} & \msd{79.76}{0.34} \\
$L_2$ Regularization
    & \msd{98.36}{0.06} & \msd{88.46}{0.17} & \msd{91.31}{0.05} & \msd{94.06}{0.44} & \msd{79.84}{0.75} \\
Label Smoothing 
    & \msd{98.57}{0.11} & \msd{89.15}{0.40} & \msd{91.40}{0.05} & \color{blue} \msd{94.70}{0.38} & \color{blue} \msd{80.45}{0.44} \\
Mixup 
    & \msd{97.80}{0.27} & \color{blue} \msd{89.50}{0.16} & \msd{91.10}{0.17} & \msd{\BF 94.88}{0.24} & \color{blue} \msd{80.92}{0.47} \\
\midrule%
$\mathrm{NC}$ Regularization
    & \msd{\BF 99.03}{0.07} & \msd{\BF 89.74}{0.17} & \msd{\BF 96.30}{0.07} & \msd{93.83}{0.18} & \msd{\BF 81.15}{0.36} \\
\bottomrule & & & & & \\
\end{tabular}
\caption{Mean test accuracies and $95\%$ confidence intervals of each method on $5$ runs.}
\label{tab:single_task}
\vspace{-20pt}
\end{table}

Finally, we evaluate \gls{nc} on single tasks, following our protocol outlined in \cref{sec:formulation} of constructing a large number of sub-tasks using only the train split and then evaluating on the test set.
We consider five different datasets:
three MNIST variants (MNIST \cite{lecun1998mnist},  FMNIST \cite{xiao2017fashion}, KMNIST \cite{clanuwat2018deep}), 
for which the learner was a 1-layer MLP with $500$ units,
and SVHN \cite{netzer2011reading} along with CIFAR-10 \cite{krizhevsky2009learning} , for which we used the ResNet-18 \cite{he2016deep} network.
To isolate the effect of the regularizers, we do not use image augmentation or manual learning rate scheduling.
Due to space constraints, we describe detailed hyperparameters and \gls{nc} architectures in the appendix.

\paragraph{Network Size}
To test whether the gains from \gls{nc} are simply from the additional parameters, we compared against larger networks in \cref{tab:capacity}.
We constructed networks with $2\times, 4\times$ capacity by jointly training multiple networks with a pooling layer.
These results show that \gls{nc} is much more effective compared to simply using a larger model, and is in fact outputting a useful approximation to $\gap$.

\paragraph{Effect of Dataset Size}
We investigated whether \gls{nc} can capture the effect of dataset size $m$ on overfitting.
Using an unregularized learner on the KMNIST dataset, we measured how the gap and \gls{nc}'s estimate of it changes during task learning.
The left figure of \cref{fig:MNIST_exps} shows that overfitting occurs more severely with smaller datasets, and \gls{nc} successfully captures this trend.

\paragraph{Regularization Performance} 
We measure \gls{nc}'s effectiveness as a regularizer, comparing it to other regularization methods for classification tasks.
We consider four baselines: standard cross-entropy loss, $L_2$ loss, label smoothing \cite{szegedy2016rethinking}, and Mixup \cite{zhang2017mixup}.
Results in \cref{tab:single_task} show that \gls{nc} consistently improves test accuracy and performs similarly to modern regularization methods, even outperforming them on some tasks.
We further visualize the learning curve of a \gls{nc}-regularized task learner on the KMNIST dataset in the middle and right figures of \cref{fig:MNIST_exps}, which show that \gls{nc} accelerates training in addition to improving the final accuracy.

\section{Conclusion}
\glsreset{nc}
We proposed \gls{nc}, a meta-learning framework for predicting generalization.
We motivated \gls{nc} through a generalization bound based on estimators of the generalization gap $\gap$. 
It treats generalization as a regression problem, learning the degree to which a function will generalize to unseen test data..
Notably, regularizing with \gls{nc} consistently resulted in models with lower test loss than train loss.
\gls{nc} is capable of regularizing task learners from previously unseen architectures and hyperparameters.
Additionally, our experiments demonstrate that \gls{nc} can learn to learn in much larger tasks compared to previous meta-learning works, such as MNIST or CIFAR.

We see many exciting future directions for improvement within the \gls{nc} framework.
First, the current requirement of validation data ($X_\mathrm{te}$) limits the applicability of our model, but this requirement is not intrinsic to the framework of \gls{nc}.
In future work, we will investigate architectures and training schemes that allow for the accurate prediction of $\gap$ from data and predictions from training data.
We also think \gls{nc} can be extended beyond supervised learning, for example, predicting the generalization gap in density estimation or self-supervised learning.
Furthermore, we are interested in scaling this approach to the ImageNet dataset \cite{deng2009imagenet} since our experiments have shown that \gls{nc} is much more scalable than other meta-learning methods.

\newpage

\section*{Broader Impact}
Safety and reliability are important desiderata for machine learning models, and these properties are even more important given the recent success of black-box models such as deep neural networks.
Our proposed approach can be applied to improve training in any regression or classification task,
and our experiments demonstrate its ability to 
(1) predict the generalization gap and 
(2) improve test loss when used as a regularizer.
\gls{nc}'s data-driven prediction of the generalization gap can serve as an approximate guarantee for safety-critical problems.
Furthermore, future extensions of \gls{nc} may enable previously impossible tasks since \gls{nc} was particularly effective in settings where conventional learners overfitted.

\begin{ack}
This work was supported by Engineering Research Center Program through the National Research Foundation of Korea (NRF) funded by the Korean Government MSIT (NRF-2018R1A5A1059921), Institute of Information \& communications Technology Planning \& Evaluation (IITP) grant funded by the Korea government (MSIT) (No.2019-0-00075), IITP grant funded by the Korea government(MSIT) (No.2017-0-01779, XAI) and the grant  funded  by  2019  IT Promotion  fund
(Development of AI based Precision Medicine Emergency System) of the Korea government (Ministry of Science and ICT).
\end{ack}

\bibliography{master.bib} 

\begin{thebibliography}{43}
\providecommand{\natexlab}[1]{#1}
\providecommand{\url}[1]{\texttt{#1}}
\expandafter\ifx\csname urlstyle\endcsname\relax
  \providecommand{\doi}[1]{doi: #1}\else
  \providecommand{\doi}{doi: \begingroup \urlstyle{rm}\Url}\fi

\bibitem[Krizhevsky et~al.(2012)Krizhevsky, Sutskever, and
  Hinton]{krizhevsky2012imagenet}
Alex Krizhevsky, Ilya Sutskever, and Geoffrey~E Hinton.
\newblock Imagenet classification with deep convolutional neural networks.
\newblock In \emph{Advances in neural information processing systems}, pages
  1097--1105, 2012.

\bibitem[Silver et~al.(2017)Silver, Hubert, Schrittwieser, Antonoglou, Lai,
  Guez, Lanctot, Sifre, Kumaran, Graepel, et~al.]{silver2017mastering}
David Silver, Thomas Hubert, Julian Schrittwieser, Ioannis Antonoglou, Matthew
  Lai, Arthur Guez, Marc Lanctot, Laurent Sifre, Dharshan Kumaran, Thore
  Graepel, et~al.
\newblock Mastering chess and shogi by self-play with a general reinforcement
  learning algorithm.
\newblock \emph{arXiv preprint arXiv:1712.01815}, 2017.

\bibitem[Jiang et~al.(2019)Jiang, Neyshabur, Mobahi, Krishnan, and
  Bengio]{jiang2019fantastic}
Yiding Jiang, Behnam Neyshabur, Hossein Mobahi, Dilip Krishnan, and Samy
  Bengio.
\newblock Fantastic generalization measures and where to find them.
\newblock \emph{arXiv preprint arXiv:1912.02178}, 2019.

\bibitem[Keskar et~al.(2016)Keskar, Mudigere, Nocedal, Smelyanskiy, and
  Tang]{keskar2016large}
Nitish~Shirish Keskar, Dheevatsa Mudigere, Jorge Nocedal, Mikhail Smelyanskiy,
  and Ping Tak~Peter Tang.
\newblock On large-batch training for deep learning: Generalization gap and
  sharp minima.
\newblock \emph{arXiv preprint arXiv:1609.04836}, 2016.

\bibitem[Liang et~al.(2017)Liang, Poggio, Rakhlin, and Stokes]{liang2017fisher}
Tengyuan Liang, Tomaso Poggio, Alexander Rakhlin, and James Stokes.
\newblock Fisher-rao metric, geometry, and complexity of neural networks.
\newblock \emph{arXiv preprint arXiv:1711.01530}, 2017.

\bibitem[Nagarajan and Kolter(2019)]{DBLP:journals/corr/abs-1901-01672}
Vaishnavh Nagarajan and J.~Zico Kolter.
\newblock Generalization in deep networks: The role of distance from
  initialization.
\newblock \emph{CoRR}, abs/1901.01672, 2019.
\newblock URL \url{http://arxiv.org/abs/1901.01672}.

\bibitem[Wu et~al.(2018)Wu, Ren, Liao, and Grosse]{wu2018understanding}
Yuhuai Wu, Mengye Ren, Renjie Liao, and Roger Grosse.
\newblock Understanding short-horizon bias in stochastic meta-optimization.
\newblock \emph{arXiv preprint arXiv:1803.02021}, 2018.

\bibitem[Huber(1992)]{huber1992robust}
Peter~J Huber.
\newblock Robust estimation of a location parameter.
\newblock In \emph{Breakthroughs in statistics}, pages 492--518. Springer,
  1992.

\bibitem[Vaswani et~al.(2017)Vaswani, Shazeer, Parmar, Uszkoreit, Jones, Gomez,
  Kaiser, and Polosukhin]{vaswani2017attention}
Ashish Vaswani, Noam Shazeer, Niki Parmar, Jakob Uszkoreit, Llion Jones,
  Aidan~N Gomez, {\L}ukasz Kaiser, and Illia Polosukhin.
\newblock Attention is all you need.
\newblock In \emph{Advances in neural information processing systems}, pages
  5998--6008, 2017.

\bibitem[Xu et~al.(2019)Xu, Ton, Kim, Kosiorek, and Teh]{xu2019metafun}
Jin Xu, Jean-Francois Ton, Hyunjik Kim, Adam~R Kosiorek, and Yee~Whye Teh.
\newblock Metafun: Meta-learning with iterative functional updates.
\newblock \emph{arXiv preprint arXiv:1912.02738}, 2019.

\bibitem[Garnelo et~al.(2018)Garnelo, Schwarz, Rosenbaum, Viola, Rezende,
  Eslami, and Teh]{garnelo2018neural}
Marta Garnelo, Jonathan Schwarz, Dan Rosenbaum, Fabio Viola, Danilo~J Rezende,
  SM~Eslami, and Yee~Whye Teh.
\newblock Neural processes.
\newblock \emph{arXiv preprint arXiv:1807.01622}, 2018.

\bibitem[Vapnik(1999)]{vapnik1999overview}
Vladimir~N Vapnik.
\newblock An overview of statistical learning theory.
\newblock \emph{IEEE transactions on neural networks}, 10\penalty0
  (5):\penalty0 988--999, 1999.

\bibitem[McAllester(1999)]{mcallester1999pac}
David~A McAllester.
\newblock Pac-bayesian model averaging.
\newblock In \emph{COLT}, volume~99, pages 164--170. Citeseer, 1999.

\bibitem[Dziugaite and Roy(2017)]{dziugaite2017computing}
Gintare~Karolina Dziugaite and Daniel~M Roy.
\newblock Computing nonvacuous generalization bounds for deep (stochastic)
  neural networks with many more parameters than training data.
\newblock \emph{arXiv preprint arXiv:1703.11008}, 2017.

\bibitem[Zhou et~al.(2018)Zhou, Veitch, Austern, Adams, and
  Orbanz]{zhou2018non}
Wenda Zhou, Victor Veitch, Morgane Austern, Ryan~P Adams, and Peter Orbanz.
\newblock Non-vacuous generalization bounds at the imagenet scale: a
  pac-bayesian compression approach.
\newblock \emph{arXiv preprint arXiv:1804.05862}, 2018.

\bibitem[Neyshabur et~al.(2015)Neyshabur, Tomioka, and
  Srebro]{neyshabur2015norm}
Behnam Neyshabur, Ryota Tomioka, and Nathan Srebro.
\newblock Norm-based capacity control in neural networks.
\newblock In \emph{Conference on Learning Theory}, pages 1376--1401, 2015.

\bibitem[Bartlett et~al.(2017)Bartlett, Foster, and
  Telgarsky]{bartlett2017spectrally}
Peter~L Bartlett, Dylan~J Foster, and Matus~J Telgarsky.
\newblock Spectrally-normalized margin bounds for neural networks.
\newblock In \emph{Advances in Neural Information Processing Systems}, pages
  6240--6249, 2017.

\bibitem[Jiang et~al.(2018)Jiang, Krishnan, Mobahi, and
  Bengio]{jiang2018predicting}
Yiding Jiang, Dilip Krishnan, Hossein Mobahi, and Samy Bengio.
\newblock Predicting the generalization gap in deep networks with margin
  distributions.
\newblock \emph{arXiv preprint arXiv:1810.00113}, 2018.

\bibitem[Yak et~al.(2019)Yak, Gonzalvo, and Mazzawi]{yak2019towards}
Scott Yak, Javier Gonzalvo, and Hanna Mazzawi.
\newblock Towards task and architecture-independent generalization gap
  predictors.
\newblock \emph{arXiv preprint arXiv:1906.01550}, 2019.

\bibitem[Unterthiner et~al.(2020)Unterthiner, Keysers, Gelly, Bousquet, and
  Tolstikhin]{unterthiner2020predicting}
Thomas Unterthiner, Daniel Keysers, Sylvain Gelly, Olivier Bousquet, and Ilya
  Tolstikhin.
\newblock Predicting neural network accuracy from weights.
\newblock \emph{arXiv preprint arXiv:2002.11448}, 2020.

\bibitem[Lee and Choi(2018)]{Lee2018}
Yoonho Lee and Seungjin Choi.
\newblock Gradient-based meta-learning with learned layerwise metric and
  subspace.
\newblock 2018.

\bibitem[Thrun and Pratt(1998)]{thrun1998learning}
Sebastian Thrun and Lorien Pratt.
\newblock Learning to learn: Introduction and overview.
\newblock In \emph{Learning to learn}, pages 3--17. Springer, 1998.

\bibitem[Schmidhuber et~al.(1996)Schmidhuber, Zhao, and
  Wiering]{schmidhuber1996simple}
Juergen Schmidhuber, Jieyu Zhao, and MA~Wiering.
\newblock Simple principles of metalearning.
\newblock \emph{Technical report IDSIA}, 69:\penalty0 1--23, 1996.

\bibitem[Ravi and Larochelle(2016)]{ravi2016optimization}
Sachin Ravi and Hugo Larochelle.
\newblock Optimization as a model for few-shot learning.
\newblock 2016.

\bibitem[Snell et~al.(2017)Snell, Swersky, and Zemel]{snell2017prototypical}
Jake Snell, Kevin Swersky, and Richard Zemel.
\newblock Prototypical networks for few-shot learning.
\newblock In \emph{Advances in Neural Information Processing Systems}, pages
  4077--4087, 2017.

\bibitem[Finn et~al.(2017)Finn, Abbeel, and Levine]{finn2017model}
Chelsea Finn, Pieter Abbeel, and Sergey Levine.
\newblock Model-agnostic meta-learning for fast adaptation of deep networks.
\newblock \emph{arXiv preprint arXiv:1703.03400}, 2017.

\bibitem[Kim et~al.(2018)Kim, Yoon, Dia, Kim, Bengio, and Ahn]{kim2018bayesian}
Taesup Kim, Jaesik Yoon, Ousmane Dia, Sungwoong Kim, Yoshua Bengio, and Sungjin
  Ahn.
\newblock Bayesian model-agnostic meta-learning.
\newblock \emph{arXiv preprint arXiv:1806.03836}, 2018.

\bibitem[Balaji et~al.(2018)Balaji, Sankaranarayanan, and
  Chellappa]{balaji2018metareg}
Yogesh Balaji, Swami Sankaranarayanan, and Rama Chellappa.
\newblock Metareg: Towards domain generalization using meta-regularization.
\newblock In \emph{Advances in Neural Information Processing Systems}, pages
  998--1008, 2018.

\bibitem[Lake et~al.(2015)Lake, Salakhutdinov, and Tenenbaum]{lake2015human}
Brenden~M Lake, Ruslan Salakhutdinov, and Joshua~B Tenenbaum.
\newblock Human-level concept learning through probabilistic program induction.
\newblock \emph{Science}, 350\penalty0 (6266):\penalty0 1332--1338, 2015.

\bibitem[Li et~al.(2018)Li, Xu, Taylor, Studer, and
  Goldstein]{li2018visualizing}
Hao Li, Zheng Xu, Gavin Taylor, Christoph Studer, and Tom Goldstein.
\newblock Visualizing the loss landscape of neural nets.
\newblock In \emph{Advances in Neural Information Processing Systems}, pages
  6389--6399, 2018.

\bibitem[LeCun(1998)]{lecun1998mnist}
Yann LeCun.
\newblock The mnist database of handwritten digits.
\newblock \emph{http://yann. lecun. com/exdb/mnist/}, 1998.

\bibitem[Xiao et~al.(2017)Xiao, Rasul, and Vollgraf]{xiao2017fashion}
Han Xiao, Kashif Rasul, and Roland Vollgraf.
\newblock Fashion-mnist: a novel image dataset for benchmarking machine
  learning algorithms.
\newblock \emph{arXiv preprint arXiv:1708.07747}, 2017.

\bibitem[Clanuwat et~al.(2018)Clanuwat, Bober-Irizar, Kitamoto, Lamb, Yamamoto,
  and Ha]{clanuwat2018deep}
Tarin Clanuwat, Mikel Bober-Irizar, Asanobu Kitamoto, Alex Lamb, Kazuaki
  Yamamoto, and David Ha.
\newblock Deep learning for classical japanese literature, 2018.

\bibitem[Netzer et~al.(2011)Netzer, Wang, Coates, Bissacco, Wu, and
  Ng]{netzer2011reading}
Yuval Netzer, Tao Wang, Adam Coates, Alessandro Bissacco, Bo~Wu, and Andrew~Y
  Ng.
\newblock Reading digits in natural images with unsupervised feature learning.
\newblock 2011.

\bibitem[Krizhevsky et~al.(2009)Krizhevsky, Hinton,
  et~al.]{krizhevsky2009learning}
Alex Krizhevsky, Geoffrey Hinton, et~al.
\newblock Learning multiple layers of features from tiny images.
\newblock 2009.

\bibitem[He et~al.(2016)He, Zhang, Ren, and Sun]{he2016deep}
Kaiming He, Xiangyu Zhang, Shaoqing Ren, and Jian Sun.
\newblock Deep residual learning for image recognition.
\newblock In \emph{Proceedings of the IEEE conference on computer vision and
  pattern recognition}, pages 770--778, 2016.

\bibitem[Szegedy et~al.(2016)Szegedy, Vanhoucke, Ioffe, Shlens, and
  Wojna]{szegedy2016rethinking}
Christian Szegedy, Vincent Vanhoucke, Sergey Ioffe, Jon Shlens, and Zbigniew
  Wojna.
\newblock Rethinking the inception architecture for computer vision.
\newblock In \emph{Proceedings of the IEEE conference on computer vision and
  pattern recognition}, pages 2818--2826, 2016.

\bibitem[Zhang et~al.(2017)Zhang, Cisse, Dauphin, and
  Lopez-Paz]{zhang2017mixup}
Hongyi Zhang, Moustapha Cisse, Yann~N Dauphin, and David Lopez-Paz.
\newblock mixup: Beyond empirical risk minimization.
\newblock \emph{arXiv preprint arXiv:1710.09412}, 2017.

\bibitem[Deng et~al.(2009)Deng, Dong, Socher, Li, Li, and
  Fei-Fei]{deng2009imagenet}
Jia Deng, Wei Dong, Richard Socher, Li-Jia Li, Kai Li, and Li~Fei-Fei.
\newblock Imagenet: A large-scale hierarchical image database.
\newblock In \emph{2009 IEEE conference on computer vision and pattern
  recognition}, pages 248--255. Ieee, 2009.

\bibitem[Dvoretzky et~al.(1956)Dvoretzky, Kiefer, and
  Wolfowitz]{dvoretzky1956asymptotic}
Aryeh Dvoretzky, Jack Kiefer, and Jacob Wolfowitz.
\newblock Asymptotic minimax character of the sample distribution function and
  of the classical multinomial estimator.
\newblock \emph{The Annals of Mathematical Statistics}, pages 642--669, 1956.

\bibitem[Massart(1990)]{massart1990tight}
Pascal Massart.
\newblock The tight constant in the dvoretzky-kiefer-wolfowitz inequality.
\newblock \emph{The annals of Probability}, pages 1269--1283, 1990.

\bibitem[Kosorok(2007)]{kosorok2007introduction}
Michael~R Kosorok.
\newblock \emph{Introduction to empirical processes and semiparametric
  inference}.
\newblock Springer Science \& Business Media, 2007.

\bibitem[Srivastava et~al.(2014)Srivastava, Hinton, Krizhevsky, Sutskever, and
  Salakhutdinov]{srivastava2014dropout}
Nitish Srivastava, Geoffrey Hinton, Alex Krizhevsky, Ilya Sutskever, and Ruslan
  Salakhutdinov.
\newblock Dropout: a simple way to prevent neural networks from overfitting.
\newblock \emph{The journal of machine learning research}, 15\penalty0
  (1):\penalty0 1929--1958, 2014.

\end{thebibliography}

\appendix
\section{Proof of Motivating Bound}
We first invoke the following lemma which relates the empirical and true cumulative distribution functions of i.i.d. random variables.
\begin{lemma}[Dvoretzky–Kiefer–Wolfowitz Inequality]
\label{lemma:DKW}
Let $X_1, \ldots, X_n$ be i.i.d. random variables with cumulative distribution function (CDF) $F(\cdot)$.
Denote the associated empirical CDF as $F_n(x) \triangleq \frac{1}{n} \sum_{i=1}^n \mathbf{1}_{\{X_i \leq x\}}$.
The following inequality holds for all $x$ w.p. $\geq 1 - \delta$:
\begin{align}
  \left|F_n(x) - F(x) \right| \leq \sqrt{\frac{\log \frac{2}{\delta}}{2n}}.
\end{align}
\end{lemma}
\begin{proof}
  We omit the proof.
  The original theorem appears in \cite{dvoretzky1956asymptotic}
  and was refined by \cite{massart1990tight}.
  This two-sided version appears in \cite{kosorok2007introduction}.
\end{proof}

\begin{proposition}
\label{thm:nc_gen}
Let $D_\calH$ be a distribution of hypotheses,
and let $f: \calZ^m \times \calH \rightarrow \Real$ be any function of the training set and hypothesis.
Let $D_\Delta$ denote the distribution of $\gap(h) - f(S, h)$ where $h \sim D_\calH$,
and let $\Delta_1, \ldots, \Delta_n$ be i.i.d. copies of $D_\Delta$.
The following holds for all $\epsilon > 0$:
\begin{align} \label{eq:NC_thm}
  \Prob \left[ 
  \left| \Ltrue(h) - \Lemp(h) \right|  \leq f(S, h) + \epsilon
  \right]
  \geq 1 - \frac{ \left| \{i | \Delta_i > \epsilon \} \right| }{n} 
      - 2\sqrt{\frac{\log \frac{2}{\delta}}{2n}}.
\end{align}
\end{proposition}

\begin{proof}
Let $F(x), F_n(x)$ be the CDF and empirical CDF of $\Delta$, respectively.
\begin{align}
  \Prob \left( |\Lemp(h) + \textrm{NC}_S(h)  - \Ltrue(h)| > \epsilon \right)
  = \Prob_{\Delta \sim p_{NC}} \left( |\Delta| > \epsilon \right) 
  = F(\epsilon) - F(-\epsilon).
\end{align}
By \cref{lemma:DKW}, the following holds with probability $\geq 1 - \delta$:
\begin{align}
  F(\epsilon) - F(-\epsilon)
  \leq F_n(\epsilon) - F_n(-\epsilon) + 2\sqrt{\frac{\log \frac{2}{\delta}}{2n}}
  = \frac{n_\epsilon}{n} + 2\sqrt{\frac{\log \frac{2}{\delta}}{2n}}. 
\end{align}
\end{proof}

\newpage
\section{Additional Experiments}
\begin{table}[t]
\centering 
\begin{tabular}{l ccccc}
\toprule%
Steps & {1} & {2} & {4} & {8} & {16} \\
\midrule%
No regularization & 4.17 & 4.04 & 4.05 & 4.04 & 4.05 \\
$L_1 (\lambda=10.0)$ & 4.21 & 4.26 & 4.26 & 4.25 & 4.25 \\
$L_1 (\lambda=1.0)$ & 4.08 & 4.00 & 3.98 & 4.03 & 4.13 \\
$L_1 (\lambda=0.1)$ & 4.07 & 3.98 & 3.95 & 4.01 & 4.12 \\
$L_1 (\lambda=0.01)$ & 4.08 & 3.98 & 3.95 & 4.04 & 4.16 \\
$L_2 (\lambda=10.0)$ & 4.10 & 4.11 & 4.17 & 4.22 & 4.32 \\
$L_2 (\lambda=1.0)$ & 4.08 & 3.98 & 3.94 & 3.96 & 4.00 \\
$L_2 (\lambda=0.1)$ & 4.07 & 3.98 & 4.03 & 4.03 & 4.14 \\
$L_2 (\lambda=0.01)$ & 4.08 & 3.98 & 3.96 & 4.04 & 4.16 \\
$L_{1, \infty} (\lambda=1.0)$ & 4.08 & 4.04 & 4.07 & 4.08 & 4.08 \\
$L_{1, \infty} (\lambda=0.1)$ & 4.07 & 3.98 & 3.95 & 4.01 & 4.12 \\
$L_{1, \infty} (\lambda=0.01)$ & 4.08 & 3.98 & 3.96 & 4.04 & 4.16 \\
$L_{3, 1.5} (\lambda=1.0)$ & 4.07 & 4.03 & 4.11 & 4.09 & 4.07 \\
$L_{3, 1.5} (\lambda=0.1)$ & 4.08 & 3.99 & 3.95 & 4.00 & 4.08 \\
$L_{3, 1.5} (\lambda=0.01)$ & 4.07 & 3.98 & 3.95 & 4.04 & 4.15 \\
Orthogonal $(\lambda=1.0)$ & 4.16 & 4.17 & 4.19 & 4.22 & 4.32 \\
Orthogonal $(\lambda=0.1)$ & 4.08 & 4.00 & 3.96 & 3.99 & 4.06 \\
Orthogonal $(\lambda=0.01)$ & 4.07 & 3.99 & 3.95 & 4.00 & 4.14 \\
Frobenius $(\lambda=1.0)$ & 4.08 & 4.01 & 4.04 & 4.13 & 4.13 \\
Frobenius $(\lambda=0.1)$ & 4.07 & 3.98 & 3.95 & 4.02 & 4.11 \\
Frobenius $(\lambda=0.01)$ & 4.08 & 3.98 & 3.96 & 4.04 & 4.15 \\
\midrule%
Dropout ($p=0.1$) & 4.08 & 3.98 & 3.96 & 4.04 & 4.15 \\
Dropout ($p=0.3$) & 4.08 & 3.98 & 3.95 & 4.02 & 4.12 \\
Dropout ($p=0.5$) & 4.08 & 3.99 & 3.95 & 4.00 & 4.07 \\
Dropout ($p=0.7$) & 4.10 & 4.00 & 3.96 & 3.98 & 4.02 \\
Dropout ($p=0.9$) & 4.17 & 4.11 & 4.09 & 4.37 & {NaN} \\
\midrule%
MetaReg & 4.04 & 3.93 & 3.89 & 3.90 & 4.00 \\
Neural Complexity & \BF 3.87 & \BF 3.60 & \BF 3.36 & \BF 3.13 & \BF 2.93 \\
\bottomrule%
\end{tabular}
\vspace{10pt}
\caption{
    Test losses of various regularization methods after a certain number of steps.
}
\label{tab:sine_regularization}
\end{table}
\begin{figure}
\begin{subfigure}{.33\textwidth}
    \centering
    \includegraphics[width=\linewidth]{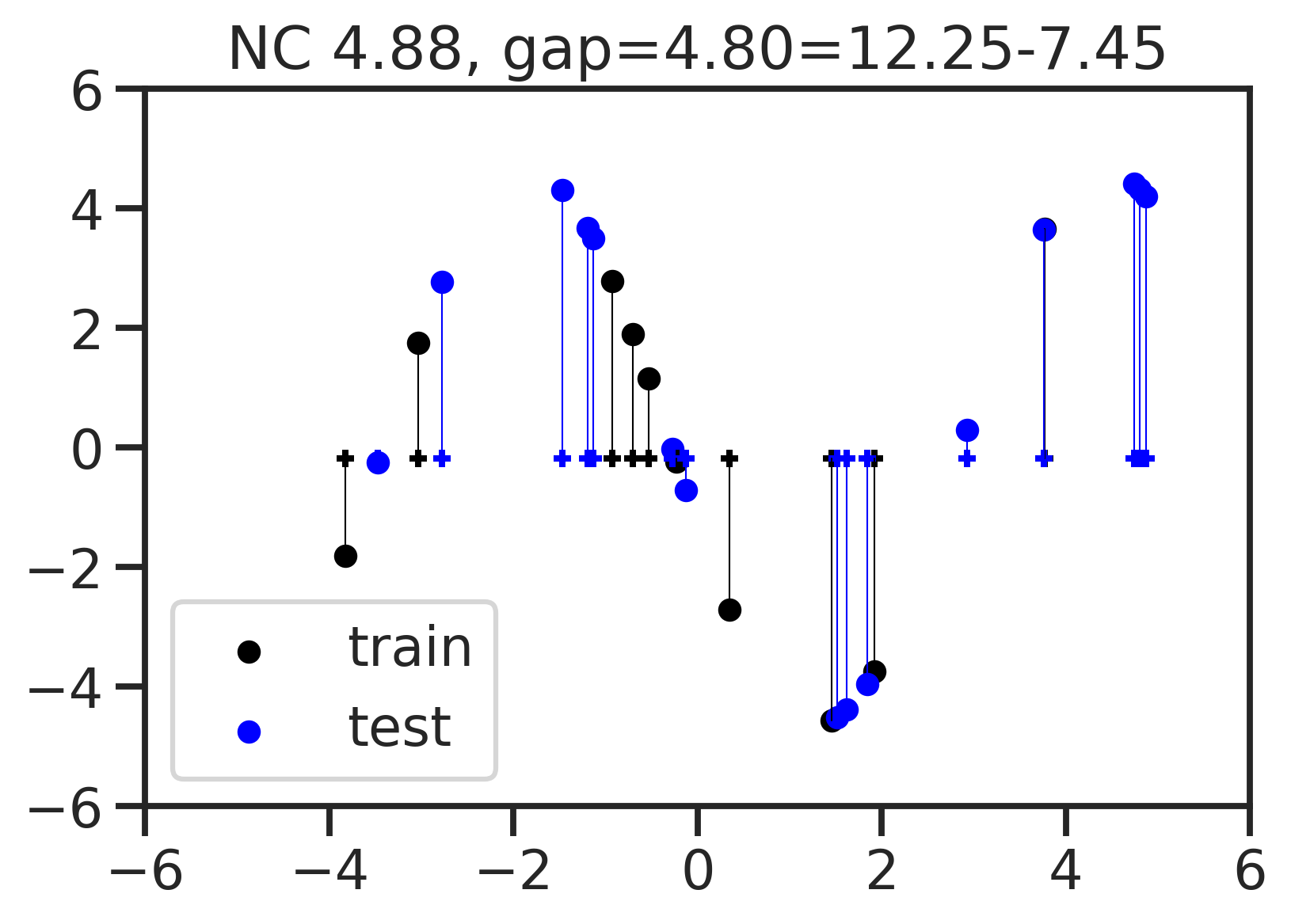}
\end{subfigure}%
\begin{subfigure}{.33\textwidth}
    \centering
    \includegraphics[width=\linewidth]{figures/nn_reg_vis/nearest_neighbor/largest_10}
\end{subfigure}%
\begin{subfigure}{.33\textwidth}
    \centering
    \includegraphics[width=\linewidth]{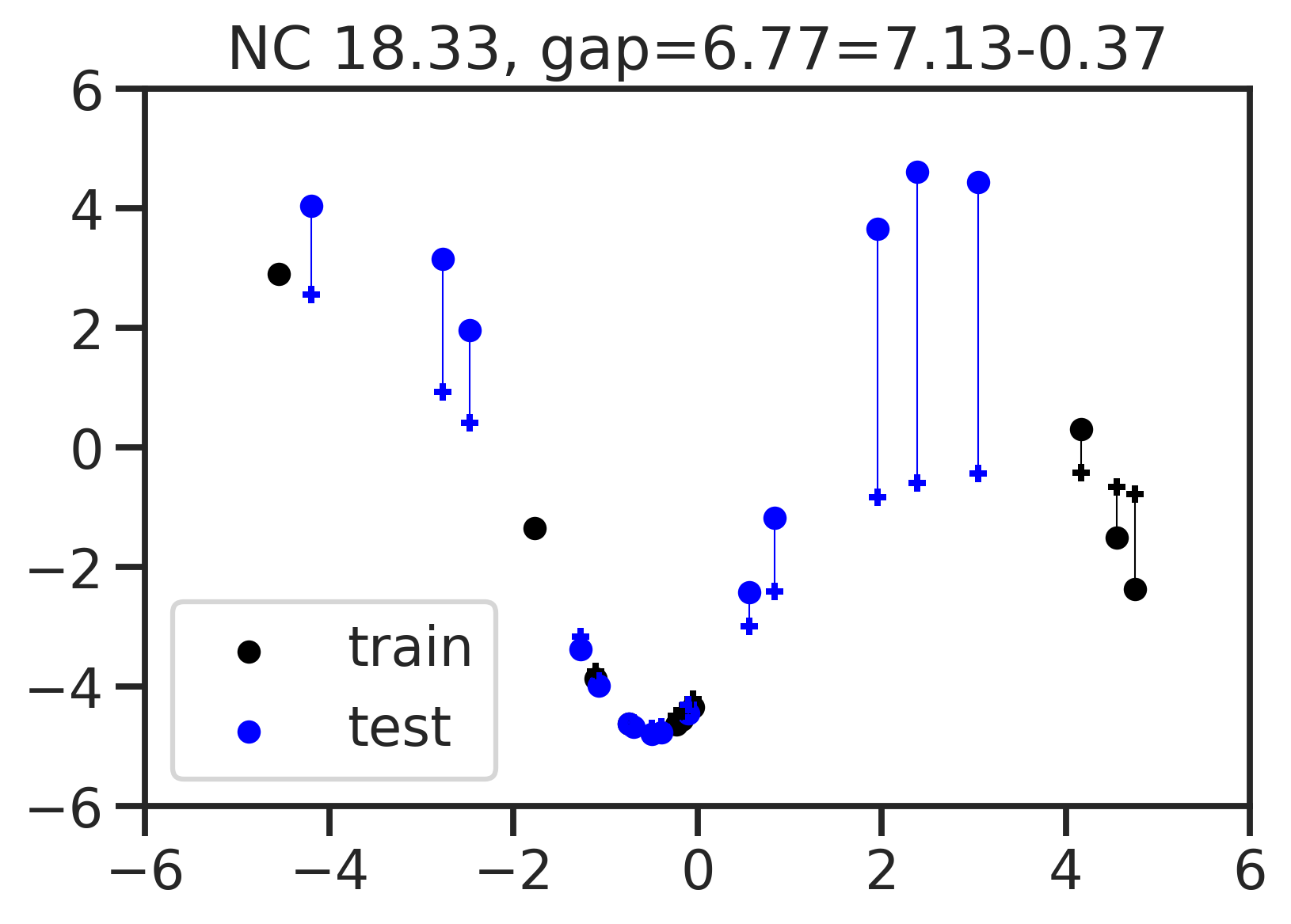}
\end{subfigure} \\
\begin{subfigure}{.33\textwidth}
    \centering
    \includegraphics[width=\linewidth]{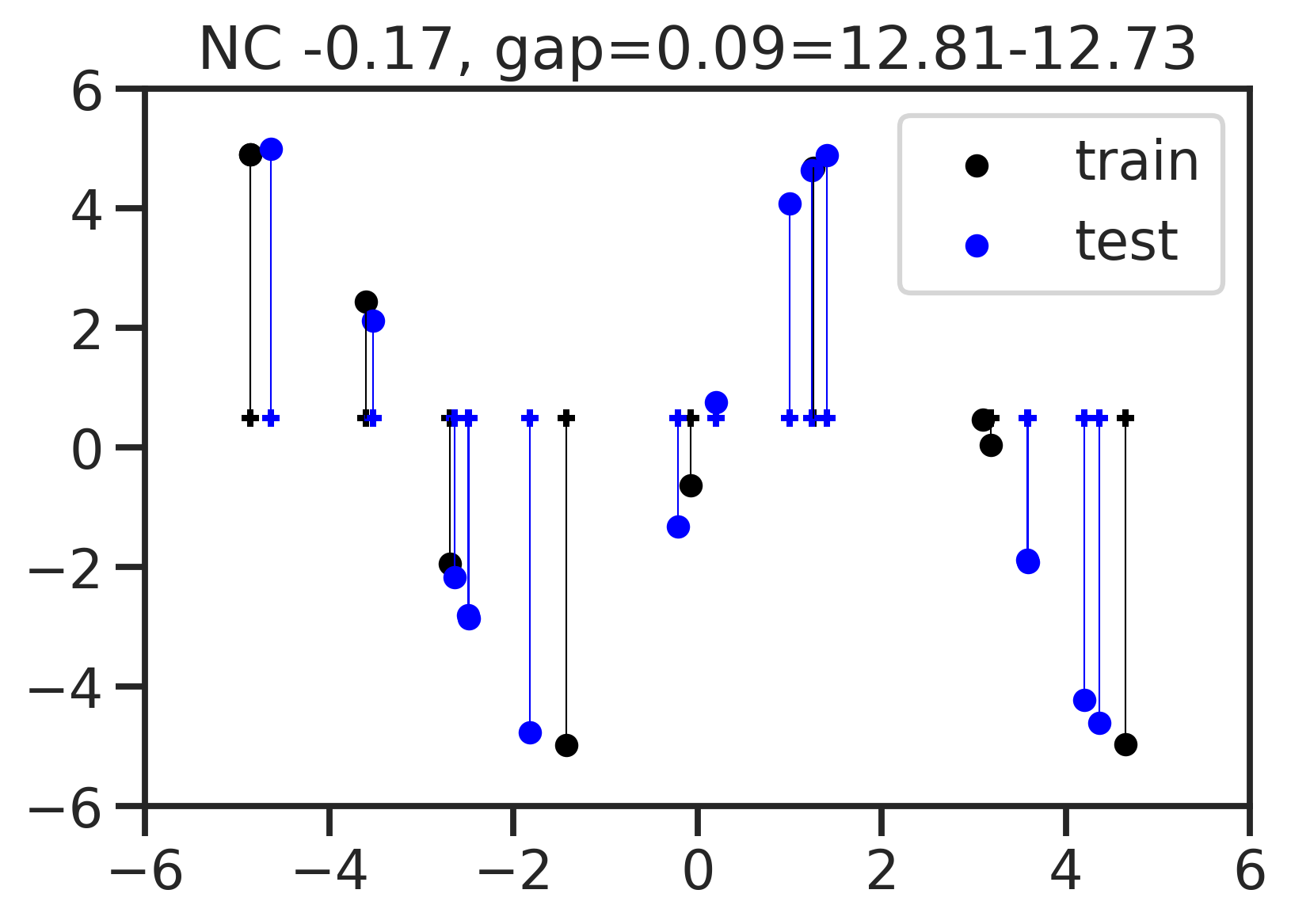}
\end{subfigure}%
\begin{subfigure}{.33\textwidth}
    \centering
    \includegraphics[width=\linewidth]{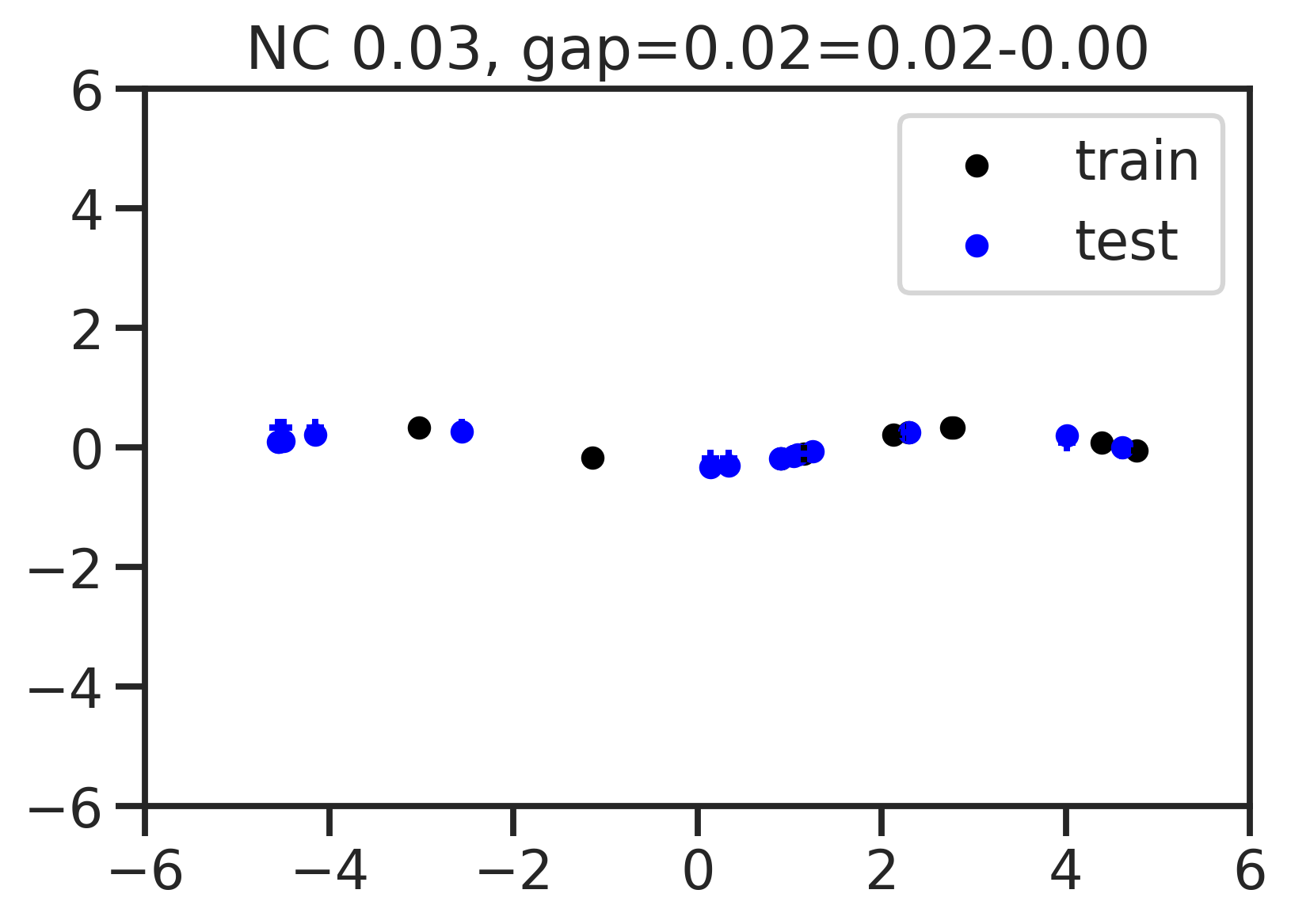}
\end{subfigure}%
\begin{subfigure}{.33\textwidth}
    \centering
    \includegraphics[width=\linewidth]{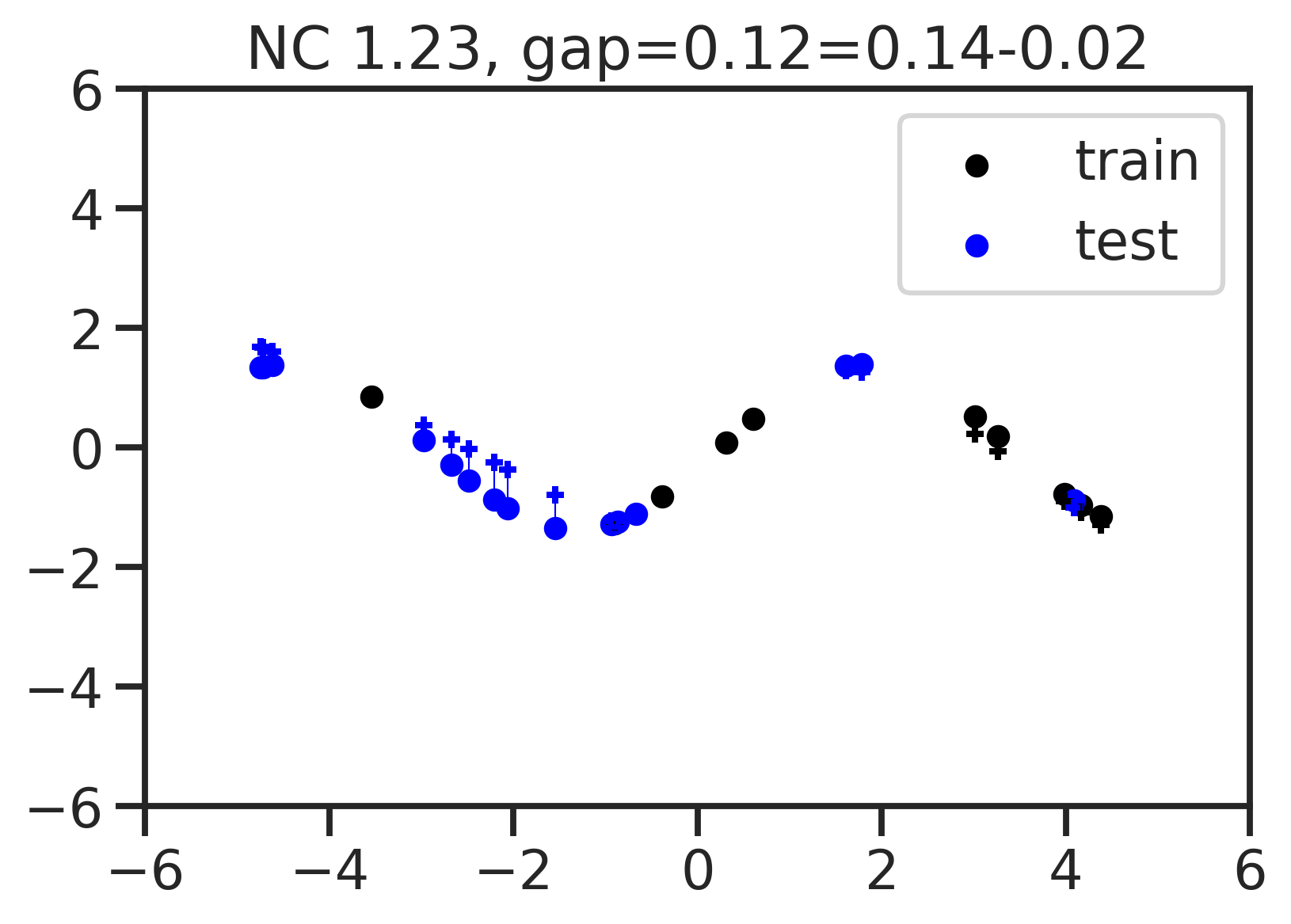}
\end{subfigure}%
\caption{
    Visualization of regression tasks.
    The x and y axes represent inputs and outputs of the task learners, respectively.
    Circles represent the targets and plus signs represent predictions.
    The NC model is trained with a neural network learner, and we evaluated on three different learners: 0-th order polynomial (left), nearest-neighbor (center), and neural networks (right).
}
\label{fig:reg_vis_app}
\end{figure}

\begin{figure} \centering
\begin{subfigure}{.33\textwidth} \centering 
\includegraphics[width=\linewidth]{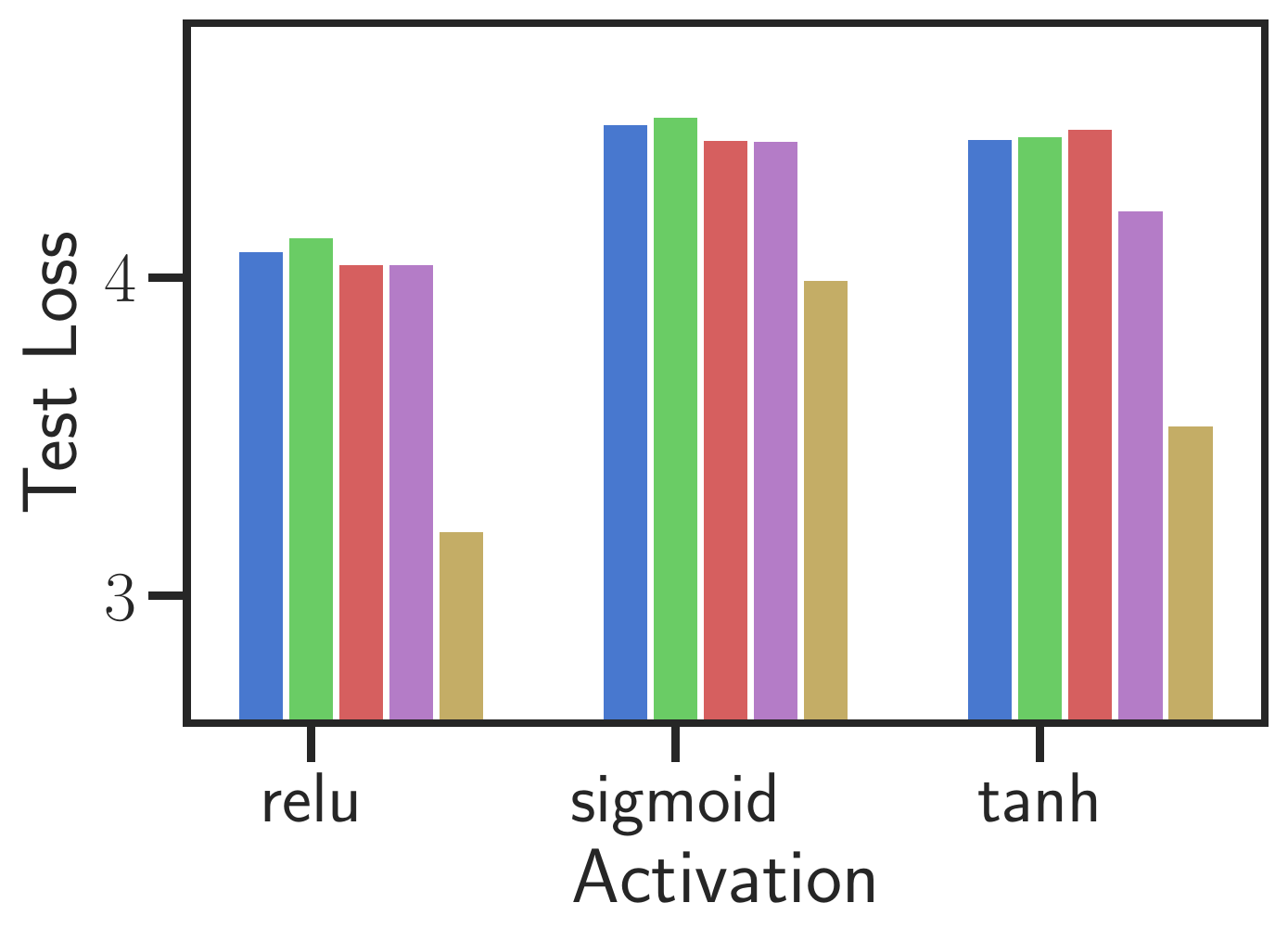}
\end{subfigure}%
\begin{subfigure}{.43\textwidth} \centering 
\includegraphics[width=\linewidth]{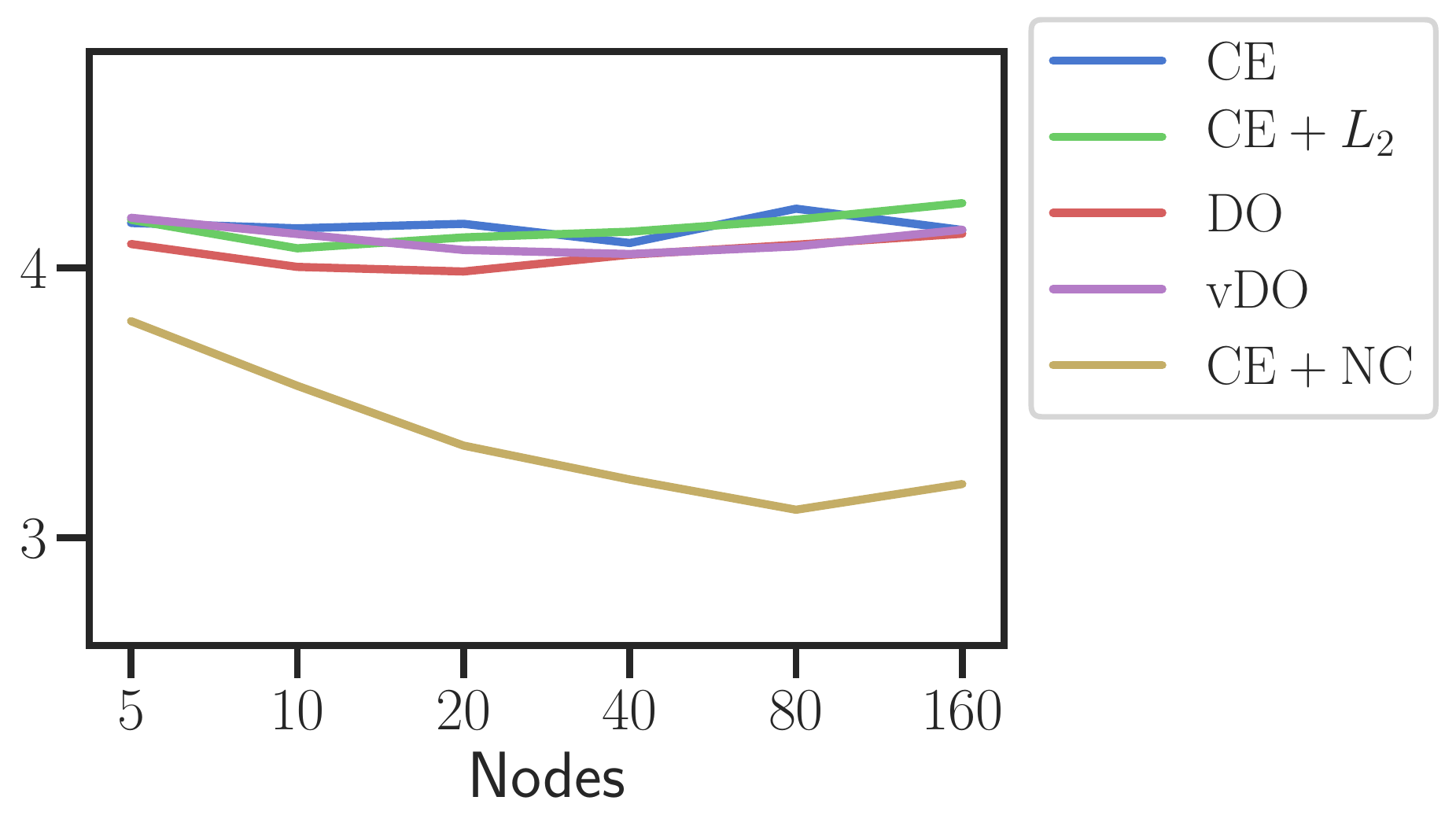}
\end{subfigure}%
\caption{
Additional experiments for out-of-distribution task learners. We additionally compare against Dropout and Variational Dropout.
}
\label{fig:reg_do}
\end{figure}

\begin{figure} \centering
\begin{minipage}{0.3\textwidth} \centering
    \includegraphics[width=.95\linewidth]{figures/3d_1shot/best/3dsurface_TrainLoss} \\
\end{minipage}%
\begin{minipage}{0.3\textwidth} \centering
    \includegraphics[width=.95\linewidth]{figures/3d_1shot/best/3dsurface_Gap}
    \includegraphics[width=.95\linewidth]{figures/3d_1shot/best/3dsurface_NC}
\end{minipage}%
\begin{minipage}{0.3\textwidth} \centering
    \includegraphics[width=.95\linewidth]{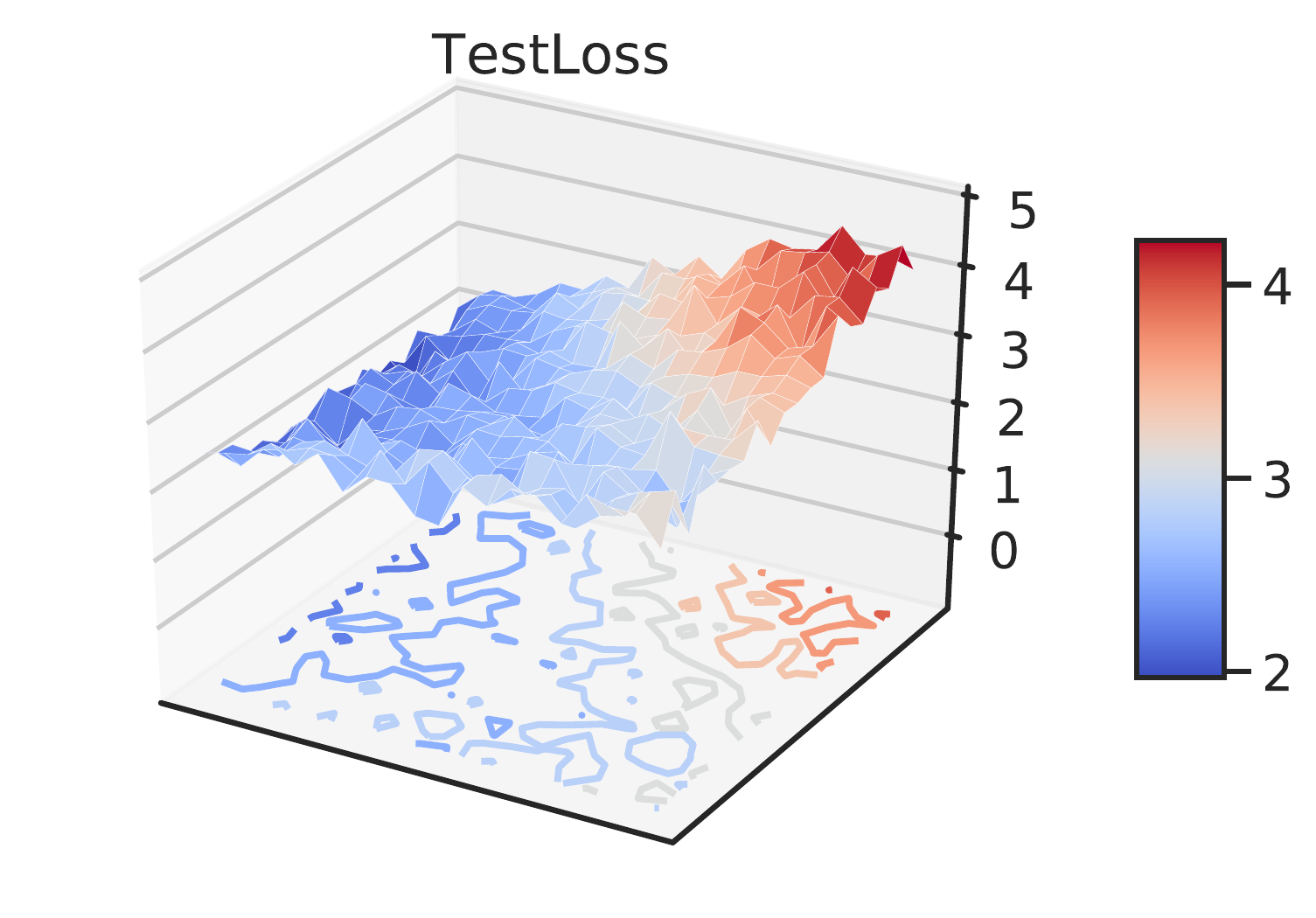}
    \includegraphics[width=.95\linewidth]{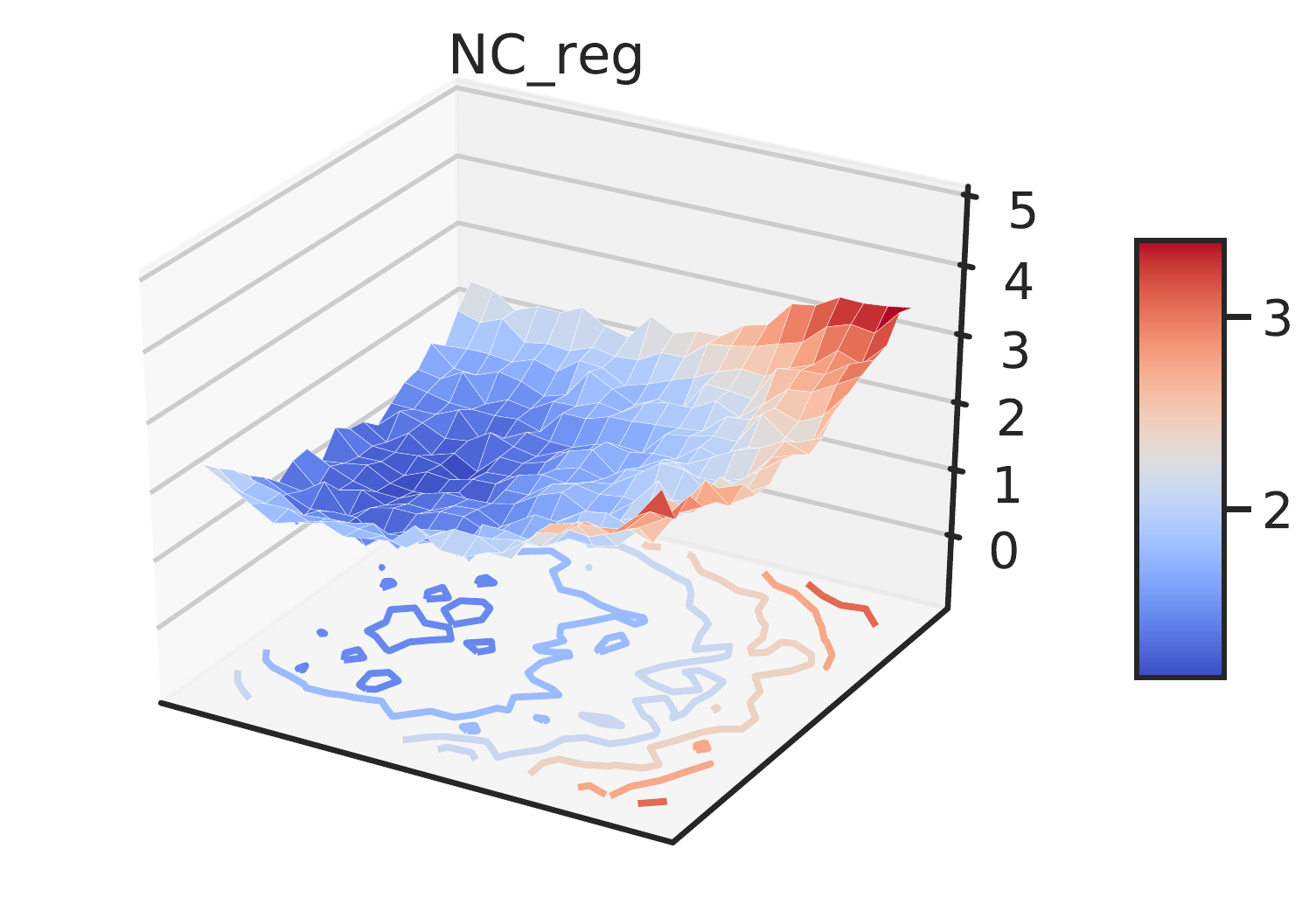}
\end{minipage}%
\caption{Visualization of loss surfaces. Best viewed zoomed in.}
\label{fig:surface_vis_app}
\end{figure}

We evaluated the performance of the following regularizers on the sinusoid regression task:
$L_1$ norm,
$L_2$ norm,
$L_{1, \infty}$ norm,
$L_{3, 1.5}$ norm,
Orthogonal constraint,
Frobenius norm,
Dropout \citep{srivastava2014dropout},
MetaReg \citep{balaji2018metareg},
and \gls{nc}.
We show performance after $\{1, 2, 4, 8, 16\}$ steps.
Results in \cref{tab:sine_regularization} show that all other baselines fail to provide guidance in this task, while \gls{nc} outperforms them by a large margin.

In \cref{fig:reg_vis_app}, we show additional visualizations of regression tasks.
This figure shows that \gls{nc} successfully captures the trend of the generalization gap even in out-of-distribution hypothesis classes.

\cref{fig:reg_do} shows additional experiments, where we additionally compare against stronger baselines (dropout and variational dropout).

\cref{fig:surface_vis_app} shows additional visualizations of loss surfaces, and reveals that the \gls{nc}-regularized loss has similar trends to that of the test loss.

\section{Experimental Details}
All experiments were ran on single GPUs (either Titan V or Titan XP) with the exception of the single-task image classification experiment, which was run on two.

These embeddings are fed into a multi-head attention layer \citep{vaswani2017attention} where queries, keys, and values are 
$Q=e_\mathrm{te}$, 
$K=e_\mathrm{tr}$,
$V=[e_\mathrm{tr}, y_\mathrm{tr}] (\in \Real^{m \times (d+1)})$,
respectively.
The output of this attention layer is a set of $m'$ items, each corresponding to a test datapoint:
\begin{align}
  \label{eq:MHA}
  f_\mathrm{att}(Q, K, V)
  = e_\mathrm{att} \in \Real^{m' \times d}.
\end{align}
Finally, these embeddings are passed through a decoding MLP network and averaged:
\begin{align}
  \mathrm{NC}(X_\mathrm{tr}, X_\mathrm{te}, Y_\mathrm{tr}, h(X_\mathrm{tr}), h(X_\mathrm{te}))
  = \frac{1}{m'} \sum_{i=1}^{m'} f_\mathrm{dec}(e_\mathrm{att})_i \in \Real.
\end{align}

\subsection{Sinusoid Regression}
\paragraph{Task Learner}
The learner was a one-layer MLP network with $40$ hidden units and ReLU activations, and was trained with vanila SGD with a learning rate of $0.01$.

\paragraph{NC Architecture}
Datapoints $x$ are encoded using an MLP encoder with $n_\mathrm{enc}$ layers, $d$-dimensional activations, and ReLU nonlinearities.
The outputs of the encoder are fed into a multi-head attention layer with $d$-dimensional activations.
The outputs of the multi-head attention layer are mean-pooled and fed into an MLP decoder with $n_\mathrm{dec}$ layers, $d$-dimensional activations, and ReLU nonlinearities.
We train \gls{nc} with batch size $\mathrm{bs}$ and the Adam optimizer with learning rate $\mathrm{lr}$.

We considered the following range of hyperparameters:
$n_\mathrm{enc} \in \{1, 2, 3\}$,
$d \in \{128, 256, 512, 1024\}$,
$n_\mathrm{dec} \in \{1, 2, 3\}$,
$\mathrm{bs} \in \{128, 256, 512, 1024\}$,
$\mathrm{lr} \in \{0.005, 0.001, 0.0005, 0.0001\}$.
We tuned these hyperparameters with a random search and ultimately used 
$n_\mathrm{enc} =3$
$d =1024$
$n_\mathrm{dec} =3$
$\mathrm{bs} =512$
$\mathrm{lr} =0.0005$.

\subsection{Classification}
\paragraph{Task Learner}
The task learner was ResNet-18 \citep{he2016deep} for the SVHN and CIFAR-10 datasets,
and an MLP with one hidden layer of $500$ nodes and ReLU nonlinearities.
To isolate the effect of the regularizers, we considered no data augmentation besides whitening.
We train all networks with SGD with a fixed learning rate and no additional learning rate scheduling.
The learning rate was $0.0001$ for ResNet-18 and $0.01$ for the MLP.

\paragraph{NC Architecture}
Datapoints $x$ are encoded using a shared CNN encoder.
The CNN architecture was the $4$-layer convolutional net in \citep{snell2017prototypical} when the task learner was an MLP, and was ResNet-18 otherwise.
We freeze all batch normalization layers inside \gls{nc}.
The outputs for only the train data is fed into a $n_\mathrm{enc}$-layer MLP followed by a stack of $n_\mathrm{self}$ self-attention layers, both with $d$-dimensional activations.
These outputs are processed by a bilinear layer,
and all outputs are fed into a multi-head attention layer with $d$-dimensional activations.
The outputs of the multi-head attention layer are then fed into an MLP decoder with $n_\mathrm{dec}$ layers, $d$-dimensional activations, and ReLU nonlinearities.
We train \gls{nc} with batch size $\mathrm{bs}$ and the Adam optimizer with learning rate $\mathrm{lr}$.

For the MLP learners, we considered the following range of hyperparameters:
$n_\mathrm{enc} \in \{1, 2, 3\}$,
$n_\mathrm{self} \in \{1, 2, 3\}$,
$d \in \{60, 120, 240\}$,
$n_\mathrm{dec} \in \{1, 2, 3\}$,
$\mathrm{bs} \in \{4, 8, 16\}$,
$\mathrm{lr} \in \{0.0005\}$.
We tuned these hyperparameters with a random search and ultimately used 
$n_\mathrm{enc} =1$,
$n_\mathrm{self} =1$,
$d =120$,
$n_\mathrm{dec} =3$,
$\mathrm{bs} =16$,
$\mathrm{lr} =0.0005$.

For the ResNet-18 learners, we considered the following range of hyperparameters:
$n_\mathrm{enc} \in \{1, 2, 3\}$,
$n_\mathrm{self} \in \{1, 2, 3\}$,
$d \in \{200, 400, 800, 1600\}$,
$n_\mathrm{dec} \in \{1, 2, 3\}$,
$\mathrm{bs} \in \{2, 4, 8\}$,
$\mathrm{lr} \in \{0.0005\}$.
We tuned these hyperparameters with a random search and ultimately used 
$n_\mathrm{enc} =1$,
$n_\mathrm{self} =3$,
$d =400$,
$n_\mathrm{dec} =3$,
$\mathrm{bs} =4$,
$\mathrm{lr} =0.0005$.

\paragraph{Single-task Experiment Details}
We provide further details about the single-task experiments.
The datasets we considered had either $50000$ or $60000$ training datapoints.
We constructed learning tasks from such training sets by sampling $40000$ "training" datapoints and $10000$ validation datapoints.
Using such splits, we trained \gls{nc} as usual.
To scale to long learning trajectories, we trained \gls{nc} using one process, 
while simultaneously adding trajectories from a separate process on a separate GPU which only ran task learners regularized by the \gls{nc} model.
During final evaluation, we clipped \gls{nc} estimates below $-0.1$, which has the effect of ignoring \gls{nc} when it is overconfident about generalization.
We found that such clipping is critical for performance on long training runs.
\end{document}